\documentclass[twoside]{article}

\usepackage[accepted]{aistats2021}
\usepackage[dvipdfmx]{graphicx}
\usepackage{url}            
\usepackage{booktabs}       
\usepackage{amsfonts}       
\usepackage{nicefrac}       
\usepackage{microtype}      
\usepackage[round]{natbib}

\usepackage{color}
\usepackage{amsmath}
\usepackage{amssymb}
\usepackage{amsthm}
\usepackage{bm}
\usepackage{bbm}

\newcommand{\ve}[1]{\mathbf{\bm{#1}}}  
\newcommand{\m}[1]{\mathbf{\bm{#1}}}  
\newcommand{\set}[1]{\mathcal{#1}}  
\newcommand{\R}{\mathbb R}

\renewcommand{\t}{\tau}
\newcommand{\x}{{\ve x}}

\newcommand{\s}{{\ve s}}
\newcommand{\y}{{\ve y}}
\newcommand{\z}{{\ve z}}
\newcommand{\h}{{\ve h}}

\newcommand{\f}{{\ve f}}

\newcommand{\g}{{\ve g}}

\newcommand{\xt}{{\ve x_t}}
\newcommand{\xtm}{{\ve x_{t-1}}}

\newcommand{\st}{{\ve s_t}}

\newcommand{\ut}{{\ve u_t}}

\newtheorem{Theorem}{Theorem}

\newtheorem{Lemma}{Lemma}

\DeclareMathOperator{\diag}{diag}
\DeclareMathOperator{\vspan}{span}

\begin{document}

%
\runningtitle{Independent Innovation Analysis}

%

\twocolumn[

\aistatstitle{Independent Innovation Analysis for \\Nonlinear Vector Autoregressive Process}

\aistatsauthor{Hiroshi Morioka \And Hermanni H\"{a}lv\"{a} \And Aapo~Hyv\"{a}rinen }

\aistatsaddress{ RIKEN AIP \And University of Helsinki \And University of Helsinki \\ Universit\'e Paris-Saclay, Inria} ]

\begin{abstract}
The nonlinear vector autoregressive (NVAR) model provides an appealing framework 
to analyze multivariate time series obtained from a nonlinear dynamical system.
However, the innovation (or error), which plays a key role by driving the dynamics, 
is almost always assumed to be additive. Additivity greatly limits the generality of the model, 
hindering analysis of general NVAR processes which have nonlinear interactions between the innovations.
Here, we propose a new general framework called independent innovation analysis (IIA),
which estimates the innovations from completely general NVAR.
We assume mutual independence of the innovations as well as their modulation 
by an auxiliary variable (which is often taken as the time index and simply interpreted as nonstationarity).
We show that IIA guarantees the identifiability of the innovations with arbitrary nonlinearities, 
up to a permutation and component-wise invertible nonlinearities. 
We also propose three estimation frameworks depending on the type of the auxiliary variable.
We thus provide the first rigorous identifiability result for general NVAR, as well as very general tools for learning such models.
\end{abstract}

\section{INTRODUCTION}

Multivariate time series are of considerable interest in a number of domains, 
such as finance, economics, and engineering. Vector autoregressive (VAR) models 
have played a central role in capturing the dynamics hidden in such time series \citep{Sims1980}.
VAR models typically attempt to fit a multivariate time series with linear coefficients representing the 
dependencies of multivariate variables within limited number of lags, 
and \textit{innovation} (or error) representing new information (impulses) fed to the process at a given time point.
Although it has been common practice to maintain a linear functional form to achieve interpretability and tractability, 
recent studies have provided a growing body of evidence that nonlinearity often exists in time series, 
and allowing for nonlinearities can be valuable for uncovering important features of dynamics
\citep{Jeliazkov2013,Kalli2018,Koop2010,Primiceri2005,Shen2019,Terasvirta1994,Tsay1998}.
Many recent studies used a deep learning framework to model nonlinear processes in video 
\citep{Finn2016,Lotter2017,Oh2015,Srivastava2015,Villegas2017,Wichers2018}
or audio \citep{Oord2016}, for example, with neural networks.

The innovation plays a key role by driving time series, and it can have a concrete meaning, such as economic shocks in finance, external torques given to a mechanical system, or stimulation in neuroscience experiments. However,  its estimation has a serious indeterminacy 
even with linear models, if only conventional statistical assumptions are made.
To facilitate estimation, VAR typically assumes that the innovations are additive, 
multivariate Gaussian (not necessary uncorrelated), and temporally independent (or serially uncorrelated).
A well-known consequence of this is that the innovations cannot be identified:
Multiplication of such innovations by any orthogonal matrices
will not  change distribution of the observed data, which hinders their interpretation.
Some studies proposed to incorporate independent component analysis (ICA) framework 
to guarantee identifiability, by assuming mutual independence of 
non-Gaussian innovations \citep{Gomez2008,Hyvarinen2010,Lanne2017,Moneta2013}.
However, those studies assumed linear VAR models, while indeterminacy would presumably be even more serious in general nonlinear VAR (NVAR) models, in which the innovations may not be additive anymore. In fact, a serious lack of identifiability in general nonlinear cases is well-known in nonlinear ICA (NICA) \citep{Hyvarinen1999}.

We propose a novel VAR analysis framework called independent innovation analysis (IIA), 
which enables estimation of innovations hidden in unknown general NVAR. We first propose a model which allows for nonlinear interactions between innovations and observations, with very general nonlinearities.
IIA can be seen as an extension of recently proposed NICA frameworks \citep{Halva2020,Hyvarinen2016,Hyvarinen2019}, 
and guarantees the identifiability of innovations up to permutation and component-wise nonlinearities. The model assumes a certain temporal structure in the innovations, which typically takes the form of nonstationarity, but can be more general.
We propose three practical estimation methods for IIA, 
two of which are self-supervised and can be easily implemented based on ordinary neural network training,
and the remaining one uses maximum-likelihood estimation in connection with a hidden Markov model.
Our identifiability theory for NVAR is quite different from anything presented earlier, and thus it can contribute as a new general framework for NVAR process.

\section{MODEL DEFINITION}

\subsection{NVAR Model and Demixing Model}

We here assume a general NVAR model, which is first order (NVAR(1)) for simplicity:
\begin{equation}
        \xt = \ve f(\xtm, \st),  \label{eq:f}
\end{equation}
where $\ve f: \R^{2n} \rightarrow \R^{n}$ represents an NVAR (mixing) model, 
and $\xt=[x_1(t), \ldots, x_n(t)]^T$ and $\st=[s_1(t), \ldots, s_n(t)]^T$ are  
observations and innovations (or errors) of the process at time point $t$, respectively.
As with ordinary VAR, the innovations are assumed to be temporally independent (serially uncorrelated).
Importantly, this model includes potential nonlinear interaction between the observations and innovations,
unlike ordinary linear VAR models \citep{Gomez2008,Hyvarinen2010,Lanne2017,Moneta2013} 
and additive innovation nonlinear models \citep{Shen2019}.
We assume that $\ve f$ is unknown and make minimal regularity assumptions on it. 
Our goal is to estimate the innovations (latent components)
$\s$ only from the observations $\x$ obtained from the unknown NVAR process.
The model, learning algorithms, Theorems, and proofs below can be easily extended to
higher order models NVAR($p$) ($p > 1$) by replacing $\xtm$ by $[\ve x_{t-1},\ldots,\ve x_{t-p}]$.

To estimate the innovation, we propose a new framework called IIA,
which learns the inverse (demixing) of the NVAR (mixing) model from the observations in data-driven manner,
based on some statistical assumptions on the innovations.
The theory is related to ICA \citep{Hyvarinen1999ica}, which estimates a demixing
 from {\it instantaneous} mixtures of latent components, i.e., $\xt = \f_\text{ICA}(\st)$, where $\f_\text{ICA}: \R^{n} \rightarrow \R^{n}$ is usually a linear function. However, IIA includes a recurrent structure of the observations in the model (Eq.~\ref{eq:f}), which makes IIA theoretically distinct from ordinary ICA. Nevertheless, in the following we leverage the recently developed theory of NICA \citep{Hyvarinen2016,Hyvarinen2019}.

We start by transforming the NVAR model to something similar to NICA.
This leads us to consider the following augmented NVAR (mixing) model
\begin{equation}
        \begin{bmatrix} \xt \\ \xtm \end{bmatrix} = \tilde{\ve f} \left( \begin{bmatrix}\st \\ \xtm \end{bmatrix} \right)
        = \begin{bmatrix}\ve f(\xtm, \st) \\ \xtm \end{bmatrix},  \label{eq:fa}
\end{equation}
where $\tilde{\ve f}: \R^{2n} \rightarrow \R^{2n}$ is the augmented model,
which includes the original NVAR model $\ve f$ in the half of the space,
and an identity mapping of $\xtm$ in the remaining subspace.
Importantly, this augmentation does not impose any particular constraint on the original model. We only assume that this augmented model is invertible
(i.e.~bijective; while $\ve f$ itself cannot be invertible) as well as sufficiently smooth, but we do not constrain it in any other way.
The estimation of the innovation $\ve s$ can then be achieved by learning the inverse (demixing) of the augmented NVAR model $\tilde{\ve f}$:
\begin{equation}
        \begin{bmatrix} \st \\ \xtm \end{bmatrix} = \tilde{\ve g} \left( \begin{bmatrix}\xt \\ \xtm \end{bmatrix} \right) 
        = \begin{bmatrix}\ve g(\xt, \xtm) \\ \xtm \end{bmatrix},  \label{eq:ga}
\end{equation}
where $\tilde{\ve g}: \R^{2n} \rightarrow \R^{2n}$ is the augmented demixing model of the (true) augmented NVAR model $\tilde{\ve f}$, and
 $\ve g(\xt, \xtm) \in \R^n$ is the sub-space of the demixing model 
 representing a mapping from two temporally consecutive observations to the innovation at the corresponding timing.
This is simply a deduction from Eq.~\ref{eq:fa},
and does not impose any additional assumptions on the original model.

\subsection{Innovation Model with Auxiliary Variable}

The estimation of the demixing model in an unsupervised (or self-supervised) manner needs some assumptions on the innovations.
Although some studies guaranteed the identifiability 
by assuming mutual independence of the innovations in linear VAR models \citep{Hyvarinen2010,Lanne2017,Moneta2013},
it would not be enough in nonlinear cases, 
as can be seen in well-known indeterminacy of NICA with i.i.d.~components \citep{Hyvarinen1999}.
Thus, we here adopt the framework recently proposed for NICA \citep{Halva2020,Hyvarinen2016,Hyvarinen2019};
we assume that the distribution of the innovation is time-dependent,
and modulated through an observable (or unobservable, Section~\ref{sec:hmm}) auxiliary information about the innovation, represented by a random variable $\ut$ for each data point $t$.
In practice, $\ut$ can simply be time-index $t$ to represent data-point-specific modulations or a time-segment-index to represent segment-wise modulations, thus incorporating information about nonstationarity.
%
More specifically, we assume the followings:
\begin{enumerate} \def\labelenumi{A\theenumi.}
%
\item Each $s_i$ is statistically dependent on some $m$-dimensional random auxiliary variable $\ve u$,  
but conditionally independent of the other $s_j$, and has a univariate exponential family distribution conditioned on $\ve u$ (we omit data index $t$ here):\label{A2}
\begin{equation}
        p(\ve s | \ve u) = \prod_{i=1}^n \frac{Q_i(s_i)}{Z_i(\ve u)} \exp \left [ \sum_{j=1}^k q_{ij} (s_i) \lambda_{ij}(\ve u) \right ],  \label{eq:logpsi}
\end{equation}
where $Q_i$ is the base measure, $Z_i$ is the normalizing constant, $k$ is the model order,
$q_{ij}$ is the sufficient statistics, and $\lambda_{ij}(\ve u)$ is a parameter (scalar function) depending on $\ve u$.\footnote{The $k$ is assumed to be minimal, meaning that we cannot rewrite the form with a smaller $k' < k$.
The parameters are assumed that for each $i$, $(\exists (\lambda_{i1}(\ve u), \ldots, \lambda_{ik}(\ve u)) | \forall s_i, \sum_{j=1}^k q_{ij} (s_i) \lambda_{ij}(\ve u) = \text{const}) \Longrightarrow (\lambda_{i1}(\ve u), \ldots, \lambda_{ik}(\ve u)) = 0$. These conditions are required for the distribution to be strongly exponential \citep{Khemakhem2020},
which is not very restrictive, and satisfied by all the usual exponential family distributions.}
%
\end{enumerate}

This model is related to the assumption of Gaussian innovations in ordinary VAR,
but requires more specific properties represented by conditional independence and sufficient probabilistic modulation,
determined by an auxiliary variable $\ve u$. 
Note that exponential families have universal approximation capabilities, so this assumption is not very restrictive \citep{Sriperumbudur2017}.

\section{LEARNING ALGORITHMS}

Depending on the type of the auxiliary variable $\ve u$ in the innovation model (see A\ref{A2}), we can develop three learning algorithms;
The first one (IIA-GCL; Section~\ref{sec:gcl}) is for
general cases of observable $\ve u$, the second one (IIA-TCL; Section~\ref{sec:tcl}) is for
specific type of observable $\ve u$ underlying within a finite number of classes, 
and the last one (IIA-HMM; Section~\ref{sec:hmm}) is for unobservable $\ve u$ represented by hidden Markov chain.

\subsection{General Contrastive Learning Framework (IIA-GCL)}
\label{sec:gcl}

In the general case with observable and possibly continuous-valued $\ve u$, we develop a general contrastive learning (GCL) framework for IIA,
based on the recently proposed NICA framework \citep{Hyvarinen2019}.
In IIA-GCL, we train a feature extractor and a logistic regression classifier, which discriminates a real dataset
composed of the true observations of $(\xt, \xtm, \ut)$, from a version where randomization is performed on $\ve u$.
Thus we define two datasets in which a data point $t$ is written as follows, respectively:
\begin{equation}
        \tilde{\ve x}_t = (\xt, \xtm, \ut) \text{ vs. }  \tilde{\ve x}_t^\ast = (\xt, \xtm, \ve u^\ast),  \label{eq:xtilde}
\end{equation}
where $\ve u^\ast$ is a random value from the distribution of $\ve u$, 
but independent of $\xt$ and $\xtm$, created in practice by random permutation of the empirical sample of $\ve u$.
We learn a nonlinear logistic regression system using a regression function of the form
\begin{align}
        r(\tilde{\ve x}_t) = &\sum_{i=1}^n \sum_{j=1}^k \psi_{ij}(h_{i}(\xt, \xtm)) \mu_{ij}(\ut) + \phi(\xtm, \ut) \nonumber \\
         &+ \alpha(\ut) + \beta(\h(\xt, \xtm)) + \gamma(\xtm), \label{eq:gcl_r}
\end{align}
which gives the posterior probability of the first class $\tilde{\ve x}$ as $1 / (1 + \exp (-r(\tilde{\ve x}_t))$. 
The scalar-valued functions $\psi_{ij}$, $h_i$, $\mu_{ij}$,  $\phi$, $\alpha$, $\beta$, and $\gamma$ 
take some specific combinations of $\xt$, $\xtm$, and $\ut$ as input,
which are designed to match to the difference of the log-pdfs of $(\xt, \xtm, \ut)$ in the two datasets, given the innovation model Eq.~\ref{eq:logpsi} 
(see Supplementary Material~\ref{sec:app_gcl}).
The universal approximation capacity \citep{Hornik1989} is assumed for those functions;
they would typically be learned by neural networks.
This learning framework and the regression function are based on the following Theorem,
proven in Supplementary Material~\ref{sec:app_gcl}:

\begin{Theorem}
Assume the following:\vspace*{-2mm}
\begin{enumerate} 
\item We obtain observations and auxiliary variable $\ve u$ from an NVAR model (Eq.~\ref{eq:f}), whose augmented model (Eq.~\ref{eq:fa}) is invertible and sufficiently smooth.\label{GA1}
\item The latent innovations of the process are temporally independent, follow the assumption A\ref{A2} with $k \ge 2$, and the sufficient statistics $q_{ij}$ are twice differentiable.\label{GA2}
\item (Assumption of Variability) There exist $nk+1$ distinct points $\ve u_0, \ldots, \ve u_{nk}$ such that the matrix\label{GA3}
\begin{align}
        \m L = (\ve\lambda(\ve u_1) - \ve\lambda(\ve u_0), \ldots, \ve\lambda(\ve u_{nk}) - \ve\lambda(\ve u_0))
\end{align}
of size $nk \times nk$ is invertible, where $\ve\lambda(\ve u) = (\lambda_{11}(\ve u), \ldots, \lambda_{nk}(\ve u))^T \in \R^{nk}$.
\item We train a nonlinear logistic regression system with universal approximation capability to discriminate between 
$\tilde{\ve x}$ and $\tilde{\ve x}^\ast$ in Eq.~\ref{eq:xtilde} with regression function in Eq.~\ref{eq:gcl_r}.\label{GA4}
\item The augmented function $\tilde{\h}(\xt, \xtm) = [\h(\xt, \xtm), \xtm]: \R^{2n} \rightarrow \R^{2n}$ is invertible.\label{GA5}
\item The scalar functions $\psi_{ij}$ in Eq.~\ref{eq:gcl_r} are twice differentiable, and for each $i$, the following implication holds:
$(\exists \ve\theta \in \R^k | \forall y, \sum_{j=1}^k \psi_{ij} (y) \theta_j = \text{const}) \Longrightarrow \ve\theta = 0$.\label{GA6}
\end{enumerate}\vspace*{-2mm}
Then, in the limit of infinite data, $\h$ in the regression function provides a consistent estimator of the IIA model:
The functions $h_i(\xt, \xtm)$ give the independent innovations, up to permutation and scalar (component-wise) invertible transformations.
\end{Theorem}

This Theorem guarantees the convergence (consistency) of the learning algorithm. It immediately implies the identifiability of the innovations, up to a permutation and component-wise invertible nonlinearities. This kind of identifiability for innovations is stronger than any previous results in the literature. The estimation is based on the learning of nonlinear logistic regression function, and thus can be easily implemented based on ordinary neural network training.
The Assumption of Variability requires the auxiliary variable $\ve u$ 
to have a sufficiently strong and diverse effect on the distributions of the innovations.
The assumptions on the NVAR model are not too restrictive, and supposed to be satisfied in many applications.
The temporal independence of the innovations is the ordinary assumption for VAR.
The assumption~\ref{GA6} indicates that $\psi_{ij}$ are not functionally redundant;
any $\psi_{ij}$ cannot be represented by a linear combination of $\psi_{il \neq j}$.
Although the assumptions of the nonlinear functions to be trained (assumptions~\ref{GA5} and \ref{GA6}) are not trivial, we assume they are only necessary to have a rigorous theory, and immaterial in any practical implementation.

\subsection{Time-Contrastive Learning Framework (IIA-TCL)}
\label{sec:tcl}

In the special case in which $\ut$ is observable and integer within a finite number of classes $[1, T]$, 
we can also develop a TCL-based framework for the estimation \citep{Hyvarinen2016}.
This special case includes time-segment-wise stationary process in which $\ut$ represents the time segment index at time $t$.

Instead of the two-class logistic regression used in IIA-GCL, IIA-TCL uses a multinomial logistic regression (MLR) classifier for the learning. 
More specifically, we learn a nonlinear MLR using a softmax function 
which represents the posterior distribution of $\ve u$, by the form
\begin{align}
     p(\ut = \t | \xt, \xtm)= 
     \frac{\exp(\sum_{i=1}^n\sum_{j=1}^k z_{ij\t})}
    {\sum_{l=1}^{T} \exp(\sum_{i=1}^n\sum_{j=1}^k z_{ijl})}, \nonumber \\
    z_{ijl} =  w_{ijl} \psi_{ij}(h_i(\xt, \xtm)) + \phi(\xtm, \ut=l) + b_l, \label{eq:tcl_r}
\end{align}
where $w_{ij\t}, b_\t$ are the class-specific weight and bias parameters of the MLR,
and $\psi_{ij}$, $h_i$, and $\phi$ are again scalar-valued functions assumed to have the universal approximation capacity. 
This functional form is designed based on the innovation model given by Eq.~\ref{eq:logpsi}
(see Supplementary Material~\ref{sec:app_tcl}).
This learning framework and the regression function are justified on the following Theorem,
proven in Supplementary Material~\ref{sec:app_tcl}:

\begin{Theorem}
Assume the following:\vspace*{-2mm}
\begin{enumerate} 
\item We obtain observations and auxiliary variable $\ve u$ from an NVAR model (Eq.~\ref{eq:f}), whose augmented model (Eq.~\ref{eq:fa}) is invertible and sufficiently smooth.\label{TA1}
\item The latent innovations of the process are temporally independent, follow the assumption A\ref{A2} with $k \ge 2$, and the sufficient statistics $q_{ij}$ are twice differentiable.\label{TA2}
\item The auxiliary variable $\ve u$ is an integer in $[1, T]$, with $T$ the number of values it takes (classes).\label{TA3}
\item The modulation matrix of size $nk \times (T-1)$\label{TA4}
\begin{align}
        \m L = (\ve\lambda(2) - \ve\lambda(1), \ldots, \ve\lambda(T) - \ve\lambda(1))
\end{align}
has full row rank $nk$, where $\ve\lambda(\t) = (\lambda_{11}(\ve u=\t), \ldots, \lambda_{nk}(\ve u=\t))^T \in \R^{nk}$.
\item We train a multinomial logistic regression with universal approximation capability to
predict the class label (auxiliary variable) $\ut$ from $(\xt, \xtm)$ with regression function in Eq.~\ref{eq:tcl_r}.\label{TA5}
\item The augmented function $\tilde{\h}(\xt, \xtm) = [\h(\xt, \xtm), \xtm]: \R^{2n} \rightarrow \R^{2n}$ is invertible.\label{TA6}
\item The scalar functions $\psi_{ij}$ in Eq.~\ref{eq:tcl_r} are twice differentiable, and for each $i$, the following implication holds:
$(\exists \ve\theta \in \R^k | \forall y, \sum_{j=1}^k \psi_{ij} (y) \theta_j = \text{const}) \Longrightarrow \ve\theta = 0$.\label{TA7}
\end{enumerate}\vspace*{-2mm}
Then, in the limit of infinite data in each class, $\h$ in the regression function provides a consistent estimator of the IIA model:
The functions $h_i(\xt, \xtm)$ give the independent innovations, up to permutation and scalar (component-wise) invertible transformations.
\end{Theorem}

Many of the assumptions are the same as those in IIA-GCL, except for the specifics of the innovation model (assumptions~\ref{TA3} and \ref{TA4}) and the learning algorithm (assumption~\ref{TA5}).
The estimation is based on self-supervised nonlinear MLR,
and thus can be easily implemented based on ordinary neural network training, like IIA-GCL.
Although the estimation methods are different, the identifiability result implied here by IIA-TCL is the same as above by IIA-GCL. 
Note that here the limit of infinite data takes the form that each class (value of $T$) has an infinite number of data points. In practice, each class is thus required to have a sufficient number of samples, so $T$ needs to be much smaller than the total number of data points; this would be natural if $T$ is a segment index
(see Fig.~\ref{fig:sim1}b for the empirical result of this point).

\subsection{Hidden Markov Model Framework (IIA-HMM)}
\label{sec:hmm}

Next, we consider a special case where no $\ve u$ is observed, and no segmentation is imposed as in TCL. Instead, we assume the nonstationarity is described by hidden states following a discrete-time Markov model \citep{Halva2020}. 
This framework does not require $\ut$ to be observable unlike the previous two frameworks, and thus can learn the model in an ``purely unsupervised'' manner. It is essentially like TCL but the segmentation is inferred as part of the learning process.

We assume the following temporal structure for $\ve u$;
\begin{enumerate} \def\labelenumi{A\theenumi.}
\setcounter{enumi}{1}
\item The latent auxiliary variable $\ut \in \{1, \ldots, C\}$ represents a hidden random states at each time point,
and it is described by a Markov chain governed by a time-invariant transition-probability matrix $\m A \in \R^{C \times C}$,
where $A_{i,j}$ denotes the probability of transitioning from state $i$ to $j$. \label{A3}
\end{enumerate}
%

From the NVAR observation model with the hidden Markov chain $\ut$ generating the innovations for each data point $t$, the likelihood is given by,
using the probability transformation formula,
\begin{align}
        &p(\x_0, \x_1, \ldots, \x_T; \m A, \ve\theta) = p(\x_0) \prod_{t=1}^{T} \left | \m J \tilde{\g}(\xt, \xtm) \right | \nonumber \\
        &\times  \sum_{\ve u_1, \ldots, \ve u_T} \pi_{\ve u_1} p(\s_1 | \ve u_1; \ve\theta) \prod_{t=2}^{T} \m A_{\ve u_{t-1}, \ve u_{t}} p(\s_t | \ve u_t; \ve\theta) \label{eq:hmm_likelihood}
\end{align}
where $\ve\theta=\{\lambda, \g\}$, $\lambda$ denotes the parameters of the innovation model with omitting subscripts (Eq.~\ref{eq:logpsi}),
$\g$ is the demixing model, whose augmented model is $\tilde{\g}$ (Eq.~\ref{eq:ga}),
$\ve \pi = (\pi_1, \ldots, \pi_C)$ is the stationary distribution of the latent state $\ve u$,
$p(\x_0)$ is the marginal distribution of $\x_0$,
and $\m J \tilde{\g}$ denotes the Jacobian of $\tilde{\g}$.
The summation (marginalization) is taken over all possible combinations of $\ve u_1, \ldots, \ve u_T$.

Unlike the previous two frameworks which are based on self-supervised learning, 
the estimation of the model has to be done by a maximum-likelihood framework since $\ut$ is unobservable here.
For example, EM algorithm can be deployed when the innovation model was chosen from
a well-known family such that the normalizing constant is tractable.
The algorithm basically follows that of \citet{Halva2020}, with some differences
coming from the autoregressive structure in the observations;
the demixing model $\tilde{\g}$ has the augmented structure defined in Eq.~\ref{eq:ga},
and the marginal distribution model of the observation $p(\x_0)$ is required.
The E-step finds the optimal sequence of the latent states $(\ve u_1, \ldots, \ve u_T)$,
and M-step updates the parameters of the model so as to maximize the lower bound.
Since a closed-form of the update for $\g$ is not available in many cases, a gradient ascent update is taken instead.
Although the gradient of the determinant of the Jacobian $| \m J \tilde{\g} |$ is generally considered to be difficult,
recent developments of autograd packages, such as JAX, makes it possible to calculate them numerically
up to moderate dimensions \citep{Halva2020}.
Moreover, it can be computed using the recently proposed relative gradient method \citep{Gresele2020}.
The identifiability of this framework is discussed in Supplementary Material~\ref{sec:app_hmm}.

\section{EXPERIMENTS}

\subsection{Simulation~1: IIA-GCL for Artificial Dynamics with Nonstationary Innovations}
\label{sec:sim1}

\paragraph{Data Generation}
We generated data from an artificial NVAR process with nonstationary innovations.
The innovations were randomly generated from a Gaussian distribution 
by modulating its mean and standard deviation across time $t$, i.e., $\ut = t$.
The modulations were designed to be temporally smooth and continuous.
The dimensions of the observations and innovations ($n$) were 20.
As the NVAR model, we used a multilayer perceptron we call NVAR-MLP,
which takes a concatenation of $\xtm$ and $\st$ as an input, then outputs $\xt$.
The goal of this simulation is to estimate the innovations $\s$ only from the observable time series $\x$,
without knowing the parameters of the NVAR-MLP.
See Supplementary Material~\ref{sec:app_sim1} for more details of the experimental settings.

\paragraph{Training}
Considering the innovation model with $\ut = t$, 
we here used IIA-GCL for the estimation of the latent innovations.
We adopted MLPs as the nonlinear scalar functions in Eq.~\ref{eq:gcl_r}.
The nonlinear regression function was trained by back-propagation with a momentum term
so as to discriminate the real dataset from its $\ut$-randomized version. 
%
For comparison, we also applied NICA based on GCL (NICA-GCL; \citet{Hyvarinen2019}),
an NVAR with additive innovation model (AD-NVAR),
and variational autoencoder (VAE; \citet{Kingma2014}) to the same data. 

\paragraph{Result}
The IIA-GCL framework could reconstruct the innovations reasonably well even for the nonlinear mixture cases ($L > 1$) (Fig.~\ref{fig:sim1}a). 
We can see that a larger amount of data make it possible to achieve higher performance, 
and higher complexity of the NVAR model makes learning more difficult. 
AD-NVAR performed well for the linear mixture case ($L=1$) because the additive innovation model is equivalent to the general NVAR model in the linear case; however, it was much worse in the nonlinear case.
As expected, the other methods performed worse than IIA-GCL because their model did not match well to the NVAR generation model.

\begin{figure*}[h]
 \centering
 \includegraphics[width=2\columnwidth]{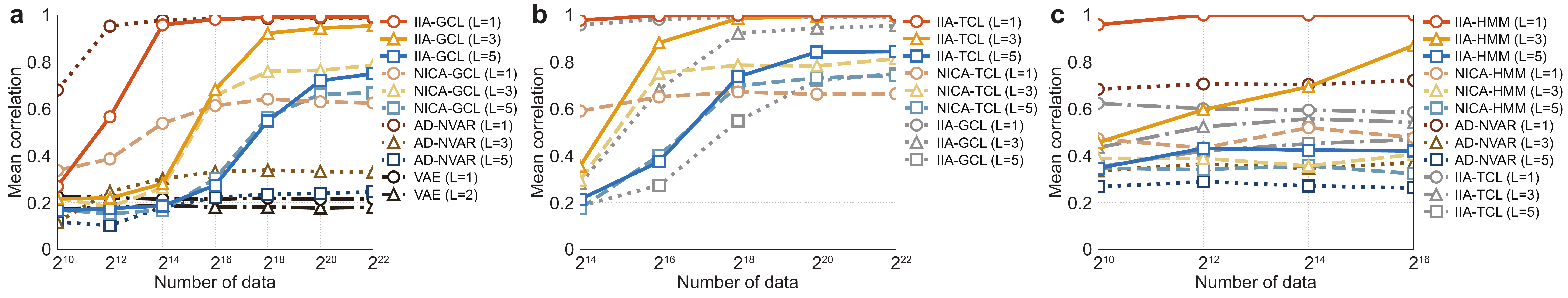}
 \caption{(Simulation) Estimation of the latent innovations from unknown artificial NVAR process by IIA. 
 (\textsf{\textbf{a}})~(Simulation~1; IIA-GCL) Mean absolute correlation coefficients between innovations and their estimates 
by IIA-GCL (solid lines), with different settings of the complexity (number of layers $L$) of the NVAR models and data points. 
For comparison: NICA based on GCL (NICA-GCL, dashed line),
NVAR with additive innovation model (AD-NVAR, dotted line), 
and variational autoencoder (VAE, dash-dot line).
IIA-GCL generally has higher correlations than the baseline methods.
 (\textsf{\textbf{b}})~(Simulation~2; IIA-TCL) Estimation performances by the IIA-TCL framework (solid lines), evaluated by the same data used in Simulation~1. 
For comparison: NICA based on TCL (NICA-TCL, dashed line)
and IIA-GCL shown in \textsf{\textbf{a}} (dotted line).
 (\textsf{\textbf{c}})~(Simulation~3; IIA-HMM) Estimation performances by the IIA-HMM framework (solid lines). 
For comparison: NICA based on HMM (NICA-HMM, dashed line),
NVAR with additive innovation model (AD-NVAR, dotted line),  
and IIA-TCL (dash-dot line).
}
 \label{fig:sim1}
\end{figure*}


\subsection{Simulation~2: IIA-TCL for Artificial Dynamics with Nonstationary Innovations}
\label{sec:sim2}

\paragraph{Training}
Next, to evaluate the IIA-TCL framework, 
we applied it to the same data used in Simulation~1.
For IIA-TCL, we first divided the time series into 256 equally-sized segments,
and used the segment label as the auxiliary variable $\ut$;
i.e., we assume that the data are segment-wise stationary,
which should be approximately true because the modulations were designed to be temporally smooth and continuous.
%
The training and evaluation methods follow those in Simulation~1.
For comparison, we also applied NICA based on TCL (NICA-TCL; \citet{Hyvarinen2016}).
See Supplementary Material~\ref{sec:app_sim2} for more details.

\paragraph{Result}
IIA-TCL performed better than NICA-TCL (Fig.~\ref{fig:sim1}b).
In addition, even though the innovation model matches IIA-GCL better than IIA-TCL
(the modulations are temporally smooth and continuous, and thus not segment-wise stationary),
IIA-TCL achieved slightly better performances than IIA-GCL (note that the performances of IIA-GCL is the same as those in Fig.~\ref{fig:sim1}a because we used the same data); this finding is consistent with the comparison of NICA-GCL and NICA-TCL by \citet{Hyvarinen2019}.
As with IIA-GCL, a larger number of data points leads to higher performance (i.e. the method seems to converge), and again, higher complexity of the NVAR models makes learning more difficult.
See also Supplementary Material~\ref{sec:app_sim2_2d} in the two dimensional case to visually see the difference of the estimation performances.

\subsection{Simulation~3: IIA-HMM for Artificial Dynamics with Hidden Markov Process}
\label{sec:sim3}

\paragraph{Data Generation}
We generated data from an artificial NVAR process with hidden Markov model.
The innovations were generated based on hidden Markov chain
with modulating the mean and the variance of Gaussian distribution for each state.
The observations were then obtained by the same method described in Simulation~1,
using the generated innovations.
The dimensions of the observations and innovations ($n$) were 5,
and the number of latent states ($C$) was 11.
See Supplementary Material~\ref{sec:app_sim3} for more details.

\paragraph{Training}
We used here EM algorithm to maximize the likelihood for estimating the 
parameters of the demixing model and the innovation process, as in \citet{Halva2020}.
For comparison, we also applied NICA-HMM \citep{Halva2020}, IIA-TCL, and AD-NVAR.

\paragraph{Result}
IIA-HMM performed better than the other baseline methods (Fig.~\ref{fig:sim1}c),
except for the most complex case ($L=5$) 
possibly because of the difficulty of the optimization due to the larger number of parameters.
The worse performances of IIA-TCL 
are likely to be due to the inconsistency between the artificial temporal segments used for the training
and the actual sequence of the hidden states, 
and also much smaller number of latent states compared to the number of the artificial segments.
AD-NVAR did not perform well even for the linear case ($L=1$) because the innovations are not necessarily marginally independent each other this time.
As with the previous frameworks, a larger number of data points leads to higher performance, and again, 
higher complexity of the NVAR models makes learning more difficult.

\subsection{Experiments on Real Brain Imaging Data}
\label{sec:exp}

\begin{figure*}[h]
 \centering
 \includegraphics[width=1.8\columnwidth]{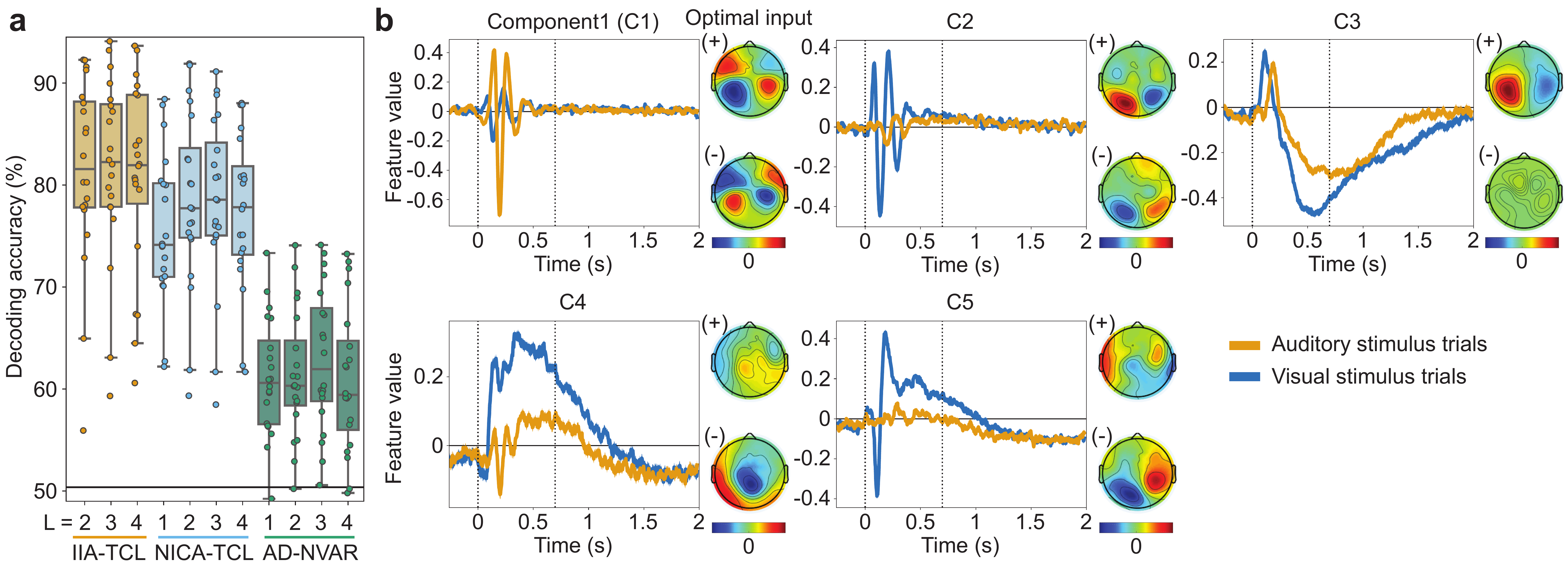}
 \caption{IIA-TCL on the electrical activity data measured by MEG from the human brain during auditory or visual stimuli of German nouns. 
 (\textsf{\textbf{a}})~Decoding accuracies of the stimulus category predicted from the innovations extracted by IIA-TCL and the other baseline methods.
 The performance was measured by one-subject-out cross-validation (OSO-CV),
 with changing the number of layers $L$ for each method. Each point represents a testing accuracy on a target subject.
 The black horizontal line indicates the chance level.
 (\textsf{\textbf{b}})~The temporal pattern and the spatial specificity of each component trained by IIA-TCL ($L=3$).
 (Left) The temporal patterns of the components averaged separately for auditory (orange) and visual trials (blue).
 0~s is the onset of the stimulus, and the latter vertical line represents the average duration of the stimuli.
 (Right) The spatial topographies of the optimal input (MEG signal; top view) which maximizes ($+$) and minimizes ($-$) the component.
 }
 \label{fig:exp}
\end{figure*}

To evaluate the applicability of IIA to real data, we applied it on
multivariate time series of electrical activities of the human brain, measured by magnetoencephalography (MEG). 
In particular, we used a dataset measured during auditory or visual presentations of words \citep{Westner2018}.
Although ICA is often used to analyze brain imaging data, relying on the assumption of mutual independence of the hidden components, 
the event-related components (such as event-related potentials; ERPs)
are not likely to be independent because they may have similar temporal patterns time-locked to the stimulation.
However, the innovations generating the components should still be independent because they would be generated by different brain sources,
which motivates us to use IIA rather than ICA (see Supplementary Material~\ref{sec:app_exp} for the details of the data and settings).

\paragraph{Data and Preprocessing}
We used a publicly available MEG dataset \citep{Westner2018}. 
Briefly, the participants were presented with a random word selected from 420 unrelated German nouns 
either visually or auditorily, for each trial. 
MEG signals were measured from twenty healthy volunteers by a 148-channel magnetometer
(219.1$\pm$22.4 trials for each subject; 2,207 auditory and 2,174 visual trials in total for all subjects).
We band-pass filtered the data between 4~Hz and 125~Hz (sampling frequency = 300~Hz).
The dimension of the data was reduced to 30 by PCA.

\paragraph{IIA Settings}
We used IIA-TCL for the training, by assuming a third-order NVAR model (NVAR(3)) and the segment-wise-stationarity of the latent innovations.
The trial data were divided into 84 equally sized segments of length of 8 samples (26.7~ms), 
and the segment label was used as the auxiliary variable $\ut$. 
The same segment labels were given across the trials; however,
considering the possible stimulus-specific dynamics of the brain, we assigned different labels for the auditory and visual trials.
In total, there are 168 segments (classes) to be discriminated by MLR.
We used MLPs for the nonlinear scalar functions (Eq.~\ref{eq:tcl_r}), 
and fixed the number of components to 5.
We fixed the time interval between two consecutive samples to 3 (10~ms). 
%

\paragraph{Evaluation Methods}
For evaluation, we performed classification of the stimulus modality (auditory or visual) by using the estimated innovations.
The classification was performed using a linear support vector machine (SVM) classifier trained on the stimulation label 
and sliding-window-averaged innovations obtained for each trial. 
The performance was evaluated by the generalizability of a classifier across subjects, i.e., one-subject-out cross-validation (OSO-CV).
For comparison, we also evaluated NICA-TCL \citep{Hyvarinen2016} and AD-NVAR(3).
We omitted $L=1$ for IIA-TCL because of the instability of training.
We visualized the spatial characteristics of each innovation component
by estimating the optimal (maximal and minimal) input $\xt$ while fixing $\ve x_{t-1:t-3}$ to zero.

\paragraph{Results}
Figure~\ref{fig:exp}a shows the decoding accuracies of the stimulus categories,
across different methods and the number of layers for each model. The performances by IIA-TCL 
with nonlinear models ($L \ge 2$) were significantly higher than the other baseline methods
($p < 0.05$; Wilcoxon signed-rank test, FDR correction),
which indicates the importance of the modeling of the MEG signals by NVAR, especially with the
nonlinear (non-additive) interactions of the innovations.

The left panels of Fig.~\ref{fig:exp}b show the temporal patterns of the innovations during the auditory and visual stimuli.
Some components have clear differences between the stimulus modalities,
which implies that those components are related to the stimulus-specific dynamics of the brain;
e.g., C1 and C2 represent auditory- and visual-relevant innovations, respectively.
Such stimulus-specificity can be also seen from the spatial characteristics of the components; 
C1 is strongly activated by the MEG signals around auditory areas of the brain, while C2 is more activated by the visual areas.
C3 seems to represent stimulus-evoked activities on the parietal region caused by both categories.
Those results show that IIA-TCL extracted reasonable components (innovations) relevant to the external stimuli
automatically from the data in a data-driven manner.

\section{DISCUSSION}
\label{sec:discussion}

IIA can be seen as a generalization of the recently proposed NICA frameworks \citep{Halva2020,Hyvarinen2016,Hyvarinen2019}, with the important difference that observations can have recurrent temporal structure. The theory strictly includes NICA as a special case, since the main assumptions can be satisfied even if the NVAR model (Eq.~\ref{eq:f}) does not actually depend on $\xtm$,  which corresponds to the instantaneous nonlinear mixture model of NICA:  $\xt = \f_\text{ICA}(\st)$. This connection can be also seen by  comparing the regression functions; by omitting the dependencies of Eqs.~\ref{eq:gcl_r} and \ref{eq:tcl_r} on $\xtm$, 
we can obtain the same algorithms as NICA (\citet{Hyvarinen2016} with k=1, and \citet{Hyvarinen2019}). This indicates that the regression functions of IIA can learn NICA models as a special case.
See Supplementary Material~\ref{sec:app_sim2_2d} for the empirical comparison in the two dimensional case.

Applying IIA on time series has some practical advantages compared to NICA.
First, autoregressive structures are generally inherent in any kinds of dynamics,
and their explicit modeling is beneficial for the estimation.
Second, innovations are usually more independent mutually than the processes generated by them, because the independence of processes implies the independence of their innovations, but not vice versa, as argued in the linear case by \citet{Hyvarinen1998}. Thus, innovations are likely to give a better fit to any model assuming independence of the latent variables.

While IIA estimates innovations from the observed time series, the NVAR model $\f$ is left unknown, unlike in ordinary VAR analyses. In practice, we can estimate $\f$ after IIA as a post-processing, by fitting a nonlinear function which outputs $\xt$ from $\xtm$ and the estimated $\st$. Since IIA guarantees the estimation of $\ve s$ up to a permutation and element-wise invertible nonlinearities, this should be possible if the model to be fitted has universal approximation capability.

\section{CONCLUSION}

We proposed independent innovation analysis (IIA) as a new general framework
to nonlinearly extract innovations hidden in a time series. 
In contrast to the common simplifying assumption of additive innovations,
IIA can deal with a general NVAR model in which innovations are not additive. Any general nonlinear interactions between the innovations and the observations are allowed.
To guarantee identifiability, IIA requires some assumptions on the innovations, 
in particular mutual independence conditionally on an auxiliary variable which also needs to modulate the distributions of the innovations. 
A typical case would be nonstationary innovations mutually independent at each time point.

We proposed three practical estimation methods. Two of them
were based on a self-supervised training of a nonlinear feature extractor by (multinomial) logistic regression. 
They can thus be easily implemented by ordinary neural network training.
The third one is a ``purely unsupervised'' framework based on maximum-likelihood estimation, specifically applicable when an 
auxiliary segmentation variable is unobservable (or in practice, we do not want to impose some simple segmentation).
The consistency of the estimation is guaranteed up to a permutation and component-wise invertible nonlinearity, 
which implies the strongest identifiability proof of general NVAR in the literature, by far.
IIA can be seen as a generalization of recently proposed NICA frameworks, and includes them as special cases.

Experiments on real brain imaging data by MEG showed distinctive components relevant to the external-stimulus categories. 
This result suggests a wide applicability of the method to different kinds of time series such as video, econometric, and biomedical data, in which innovation plays an important role.

\subsubsection*{Acknowledgements}
This research was supported in part by JSPS KAKENHI 18KK0284, 19K20355, and 19H05000, and JST PRESTO JPMJPR2028. A.H. was funded by a Fellow Position from CIFAR, the DATAIA convergence institute as part of the ``Programme d'Investissement d'Avenir" (ANR-17-CONV-0003) operated by Inria, and the Academy of Finland (project \#330482).


\bibliographystyle{plainnat}
\bibliography{references}


\normalsize
\newpage
\onecolumn
\appendix

\section*{Supplementary Materials for \textit{\textbf Independent Innovation Analysis for Nonlinear Vector Autoregressive Process}}
\addcontentsline{toc}{section}{Appendices}
\renewcommand{\thesubsection}{\Alph{subsection}}

\subsection{Proof of Theorem~1}
\label{sec:app_gcl}

The log-pdf of $\tilde{\ve x}$ is given by, using the probability transformation formula,
\begin{align}
        \log p(\tilde{\ve x}(t)) &= \log p(\xt, \xtm, \ut) \nonumber \\
        &= \log p_{\tilde{\s}}(\tilde{\ve g}(\xt, \xtm) | \ut) + \log p(\ut) + \log | \det \m J \tilde{\ve g}(\xt, \xtm) | \nonumber \\
        &= \log p_{\tilde{\s}}(\ve g(\xt, \xtm) , \xtm | \ut) + \log p(\ut) + \log | \det \m J \tilde{\ve g}(\xt, \xtm) | \nonumber \\
        &= \log p_\s(\ve g(\xt, \xtm) | \ut)  + \log p_\x(\xtm | \ut) + \log p(\ut) + \log | \det \m J \tilde{\ve g}(\xt, \xtm) | \label{eq:gcl_pxxu}
\end{align}
where $p_{\tilde{\s}}$, $p_\s$, and $p_\x$ are the conditional pdfs of $(\s, \x)$, $\s$, and $\x$, respectively,
$\m J$ denotes the Jacobian, and $s_i = g_{i}(\xt, \xtm)$ by definition.
The third equation is from the structure of the augmented demixing model (Eq.~\ref{eq:ga}), and
the last equation is from the temporal independence of $\st$ (assumption~\ref{GA2}).

By well-known theory \citep{Gutmann2012,Hastie2001}, 
after convergence of logistic regression, with infinite data and a function approximator with universal approximation capability, 
the regression function (Eq.~\ref{eq:gcl_r}) will equal the difference of the log-pdfs in the two classes:
\begin{align}
        &\sum_{i=1}^n \sum_{j=1}^k \psi_{ij}(h_{i}(\xt, \xtm)) \mu_{ij}(\ut) + \phi(\xtm, \ut) + \alpha(\ut) + \beta(\h(\xt, \xtm)) + \gamma(\xtm) \nonumber \\
        &= \log p_\s(\ve g(\xt, \xtm) | \ut) + \log p_\x(\xtm | \ut) + \log p(\ut) + \log | \det \m J \tilde{\ve g}(\xt, \xtm) |  \nonumber \\
        &- \log p_{\bar{\s}}(\g(\xt, \xtm)) - \log p_{\bar{\x}}(\xtm) - \log p(\ut) - \log | \det \m J \tilde{\ve g}(\xt, \xtm) | \nonumber \\
        &= \sum_{i=1}^{n} \left [ Q_i(g_i(\xt, \xtm)) - Z_i(\ut) + \sum_{j=1}^k q_{ij} (g_i(\xt, \xtm)) \lambda_{ij}(\ut) \right ] + \log p_\x(\xtm | \ut) \nonumber \\
        &- \log p_{\bar{\s}}(\g(\xt, \xtm)) - \log p_{\bar{\x}}(\xtm) \label{eq:gcl_eq}
\end{align}
where $p_{\bar{\s}}$ and $p_{\bar{\x}}$ are the marginal pdfs of the innovations and observations when $\ve u$ is integrated out,
and the last equation came from the conditional exponential pdf model of $\s$ (A\ref{A2}).
The Jacobians and marginals $\log p(\ve u)$ cancel out here.
Considering its factorization form and the distinctive dependency of each term on $\xt$, $\xtm$, and $\ut$, the approximation solution is possible as
\begin{align}
        \psi_{ij}(h_{i}(\xt, \xtm)) &= q_{ij} (g_i(\xt, \xtm))  \nonumber \\
        \mu_{ij}(\ut) &= \lambda_{ij}(\ut)  \nonumber \\ 
        \phi(\xtm, \ut) &= \log p_\x(\xtm | \ut) \nonumber \\
        \alpha(\ut) &= - \sum_{i=1}^{n} Z_i(\ut) \nonumber \\
        \beta(\h(\xt, \xtm)) &= \sum_{i=1}^{n} Q_i(g_i(\xt, \xtm)) - \log p_{\bar{\s}}(\g(\xt, \xtm)) \nonumber \\
        \gamma(\xtm) &= - \log p_{\bar{\x}}(\xtm). \label{eq:gcl_eq_terms}
\end{align}
Next, we have to prove that this is the only solution up to the indeterminacies given in the Theorem.
Let $\ve u_0, \ldots, \ve u_{nk}$ be the points given by assumption~\ref{GA3} in the Theorem.
We plug each of those $\ve u_l$ to obtain $nk+1$ equations. 
By collecting those equations into rows, with subtracting the first equation for $\ve u_0$ from the
remaining $nk$ equations:
\begin{align}
	\m M^T \ve\psi(\ve h(\xt, \xtm)) + \ve\phi(\xtm) + \ve\alpha
	= \m L^T \ve q(\st) +  \ve p(\xtm) + \ve z, \label{eq:gcl_eq_vec}
\end{align}
where $\m M \in \R^{nk \times nk}$ is a matrix of $\mu_{ij}(\ve u_l) - \mu_{ij}(\ve u_0)$, with the product of $i, j$ giving row index and $l$ column index,
$\m L$ is a matrix of $\lambda_{ij}(\ve u_l) - \lambda_{ij}(\ve u_0)$ given in the assumption~\ref{GA3} in the Theorem,
$\ve\psi(\ve h(\xt, \xtm)) = (\psi_{11}(h_{1}(\xt, \xtm)), \ldots, \psi_{nk}(h_{n}(\xt, \xtm)))^T$,
$\ve q(\st) = (q_{11}(s_1(t)), \ldots, q_{nk}(s_n(t)))^T$,
and the other vectors are the collection of the corresponding terms in Eq.~\ref{eq:gcl_eq} at the $nk$ points with all subtracting the one with $l=0$;
$\ve\phi(\xtm) =  (\phi(\xtm, \ve u_1), \ldots, \phi(\xtm, \ve u_{nk}))^T - \ve 1 \phi(\xtm, \ve u_0)$,
$\ve 1$ is a $nk \times 1$ vector of ones,
$\ve\alpha = (\alpha(\ve u_1), \ldots, \alpha(\ve u_{nk}))^T - \ve 1 \alpha(\ve u_0)$,
$\ve p(\xtm) = (\log p_\x(\xtm | \ve u_1), \ldots, \log p_\x(\xtm | \ve u_{nk}))^T - \ve 1 \log p_\x(\xtm | \ve u_0)$,
and $\ve z = (- \sum_{i=1}^{n} Z_i(\ve u_1), \ldots, - \sum_{i=1}^{n} Z_i(\ve u_{nk}))^T + \ve 1 \sum_{i=1}^{n} Z_i(\ve u_0)$.
In both sides of the equation, the terms not depending on $\ut$ disappeared by the subtraction with $l=0$.
Let a compound demixing-mixing function $\ve v(\st, \xtm) = \h \circ \tilde{\f}(\st, \xtm)$, 
and change variables to $\y = [\y_1, \y_2] = [\st, \xtm]$, we then have
\begin{align}
	\m M^T \ve\psi(\ve v(\y)) + \ve\phi(\y_2) + \ve\alpha
	= \m L^T \ve q(\y_1) +  \ve p(\y_2) + \ve z. \label{eq:gcl_eq_vec_v}
\end{align}

Firstly, we will show that $\m M$ is invertible.
From the definition of $\ve q(\y_1)$, its partial derivative with respect to $y_{1i}$
is $\ve q'(y_{1i}) = (0, \ldots, 0, q_{i1}'(y_{1i}), \ldots ,q_{ik}'(y_{1i}), 0, \ldots, 0)^T$.
According to Lemma~3 of \citet{Khemakhem2020}, for $y_{1i}$ which satisfies A\ref{A2},
there exist $k$ points $(\bar{y}_{1i}^1, \ldots, \bar{y}_{1i}^k)$ 
such that $(\ve q'(\bar{y}_{1i}^1), \ldots, \ve q'(\bar{y}_{1i}^k))$ are linearly independent.
By differentiating Eq.~\ref{eq:gcl_eq_vec_v} with respect to $y_{1i}$ 
and collecting their evaluations at such $k$ distinctive points for all $i$, we get
\begin{align}
	\m M^T \tilde{\m Q} = \m L^T \m Q,
\end{align}
where $\m Q  \in \R^{nk \times nk}$ is a matrix collecting $\ve q'(\bar{y}_{1i}^l)$ to the columns indexed by $(i, l)$,
and $\tilde{\m Q}$ is a collection of partial derivatives of $\ve\psi(\ve v(\y))$ evaluated at the same points.
$\m Q$ is invertible (through a combination of Lemma~3 of \citet{Khemakhem2020} and the fact that each component of $\ve q$ is univariate), and thus the right-hand side is invertible because $\m L$ is invertible as well (assumption~\ref{GA3}).
The invertibility of the right-hand side implies the invertibility of $\m M$ and $\tilde{\m Q}$.

Now, let an augmented compound demixing-mixing function $\tilde{\ve v}(\y) = [\tilde{\ve v}_1(\y), \tilde{\ve v}_2(\y)] = \tilde{\h} \circ \tilde{\f}(\y)$,
where $\tilde{\h}$ is the augmented function defined in the assumption~\ref{GA5} in the Theorem.
The $\tilde{\ve v}_1(\y)$ corresponds to $\ve v(\y)$ defined above.
Note that $\tilde{\ve v}$ is invertible because both $\tilde{\h}$ and $\tilde{\f}$ are invertible.
What we need to prove is that $\tilde{\ve v}$ is a block-wise invertible point-wise function,
in the sense that $\tilde{v}_{1i}$ is a function of only one $y_{1j_i}$ and not of any of $y_{2j_i}$, and vise versa.
This can be done by showing that the product of any two distinct partial derivatives
of any component is always zero, and the Jacobian $\m J_{\tilde{\ve v}} \in \R^{2n \times 2n}$ is block diagonal;
the upper and lower block correspond to $\y_1$ and $\y_2$ respectively.
Along with invertibility, this means that each component depends exactly on one variable of the corresponding block ($\y_1$ or $\y_2$).
Below, we show that separately for $\m J_{\ve v} \in \R^{n \times 2n}$ and $\m J_{\tilde{\ve v}_2} \in \R^{n \times 2n}$.
Firstly, this is obviously true for $\m J_{\tilde{\ve v}_2}$ because $\tilde{\ve v}_2(\y)$ is just an identity mapping of
$\y_2$ from the definitions of $\tilde{\h}$ and $\tilde{\f}$, and does not depend on $\y_1$;
the lower non-zero block of $\m J_{\tilde{\ve v}}$ is an identity matrix.
Next, we will show that for $\m J_{\ve v}$.
We differentiate Eq.~\ref{eq:gcl_eq_vec_v} with respect to $y_{c}, 1 \le c \le n$ (an element of $\y_1=\st$), 
and $y_{d}, c < d \le 2n$, and get
\begin{align}
	\m M^T \frac{\partial^2}{\partial y_c \partial y_d} \ve\psi(\ve v(\y)) = 0.
\end{align}
From the invertibility of $\m M$ and the calculation of differentials, we get
\begin{align}
	\frac{\partial^2}{\partial y_c \partial y_d} \ve\psi(\ve v(\y)) = \m\Psi(\y)^T \ve\upsilon(\y) = 0,
\end{align}
where $\m\Psi(\y) = (\ve e^{(1,1)}(y_1), \ldots, \ve e^{(1,k)}(y_1), \ldots, \ve e^{(n,1)}(y_n), \ldots, \ve e^{(n,k)}(y_n)) \in \R^{2n \times nk}$,
$\ve e^{(a,b)} = (0, \ldots, 0, \psi_{ab}'(v_a), \psi_{ab}''(v_a), 0, \ldots, 0)^T \in \R^{2n}$,
such that the non-zero entries are at indices $(2a-1, 2a)$, 
$\ve\upsilon(\y) = (v_{1}^{c,d}(\y), v_{1}^{c}(\y)v_{1}^{d}(\y), \ldots, v_{n}^{c,d}(\y), v_{n}^{c}(\y)v_{n}^{d}(\y))^T \in \R^{2n}$,
$v_i^c = \frac{\partial v_i}{\partial y_c}(\y)$, and $v_i^{c,d} = \frac{\partial^2 v_i}{\partial y_c \partial y_d}(\y)$.
From Lemma~4 and 5 of \citet{Khemakhem2020}, assumption~\ref{GA6} implies that 
$\m\Psi(\y)$ has full row rank $2n$, and thus the pseudo-inverse of $\m\Psi(\y)^T$ fulfils $\m\Psi(\y)^{+T} \m\Psi(\y)^{T} = \m I$.
We multiply the equation above from the left by such pseudo-inverse and obtain
\begin{align}
	\ve\upsilon(\y) = 0.
\end{align}
In particular, $v_{a}^{c}(\y)v_{a}^{d}(\y) = 0$ for all $1 \le a \le n$, $1 \le c \le n$, and $c < d \le 2n$.
This means that a row of $\m J_\ve v \in \R^{n \times 2n}$ at each $\y$ has either
1)~only one non-zero entry somewhere in the former half block (corresponding to the partial derivatives by $\y_1$) 
or 2)~non-zero entries only in the latter half block (corresponding to the partial derivatives by $\y_2$).
The latter case is contradictory because it means that the component $v_i$ is a function of only $\y_2 = \xtm$,
and cannot hold Eq~\ref{eq:gcl_eq_vec_v}, which right-hand side is a function of all components of $\y_1$ (and $\y_2$).
Therefore, $\m J_\ve v$ should have only one non-zero entry in the former half block for each row.
From the results of $\m J_\ve v$ and $\m J_{\tilde{\ve v}_2}$,
we deduce that $\m J_{\tilde{\ve v}}$ is a block diagonal matrix.
Now, by invertibility and continuity of $\m J_{\tilde{\ve v}}$, we deduce that the location of the non-zero entries are fixed
and do not change as a function of $\y$. This proves that $\tilde{\ve v} = \tilde{\h} \circ \tilde{\f}(\y)$ is a block-wise invertible point-wise function,
and $v_{i}$ ($=h_i(\xt, \xtm)$) is represented by only one $y_{1j_i}$ ($=s_{j_i}(t)$) up to a scalar (component-specific) invertible transformation,
and the Theorem is proven.

\subsection{Proof of Theorem~2}
\label{sec:app_tcl}

The conditional joint log-pdf of a data point $(\xt, \xtm)$ is given by, 
using the probability transformation formula,
\begin{align}
        \log p(\xt, \xtm | \ut) &= \log p_{\tilde{\s}}(\tilde{\ve g}(\xt, \xtm) | \ut) + \log | \det \m J \tilde{\ve g}(\xt, \xtm) | \nonumber \\
        &= \log p_\s(\g(\xt, \xtm) | \ut) + \log p_\x(\xtm | \ut) + \log | \det \m J \tilde{\ve g}(\xt, \xtm) | \nonumber \\
        &= \sum_{i=1}^{n} \left[ Q_i(g_i(\xt, \xtm)) - Z_i(\ut) + \sum_{j=1}^k q_{ij} (g_i(\xt, \xtm)) \lambda_{ij}(\ut) \right] \nonumber \\
        & + \log p_\x(\xtm | \ut) + \log | \det \m J \tilde{\ve g}(\xt, \xtm) |  \label{eq:tcl_pxxu}
\end{align}
where $p_{\tilde{\s}}$, $p_\s$, and $p_\x$ are the conditional pdfs of $(\s, \x)$, $\s$, and $\x$, respectively,
$\m J$ denotes the Jacobian, and $s_i = g_{i}(\xt, \xtm)$ by definition.
The second equation is from the structure of the augmented demixing model (Eq.~\ref{eq:ga})
and the temporal independence of $\s$ (assumption~\ref{TA2}), and the last equation is from the conditional exponential family model of the innovation (A\ref{A2}).
On the other hand, by applying Bayes rule on the optimal discrimination relation given by Eq.~\ref{eq:tcl_r},
after dividing all the exponential term by the one of $\tau=1$ to avoid the well-known indeterminacy of the softmax function,
\begin{align}
     \log p(\xt, \xtm | \ut=\t) &= \sum_{i=1}^n \sum_{j=1}^k (w_{ij\t} - w_{ij1}) \psi_{ij}(h_i(\xt, \xtm)) + \phi(\xtm, \ut=\t) \nonumber \\
     &-  \phi(\xtm, \ut=1) + \log p(\xt, \xtm | \ut = 1) + \alpha_\t,  \label{eq:tcl_r_bayes}
\end{align}
where $\alpha_\t = b_\t - b_1 - \log p(\ut = \t) + \log p(\ut = 1)$.
Substituting Eq.~\ref{eq:tcl_pxxu} with $\ut = 1$ into Eq.~\ref{eq:tcl_r_bayes}, we have;
 \begin{align}
     \log p(\xt, \xtm | \ut=\t) &= \sum_{i=1}^{n}\sum_{j=1}^k \left[ (w_{ij\t} - w_{ij1}) \psi_{ij}(h_i(\xt, \xtm)) + q_{ij} (g_i(\xt, \xtm)) \lambda_{ij}(\ut=1) \right]  \nonumber \\
     &+ \sum_{i=1}^{n} \left[ Q_i(g_i(\xt, \xtm)) - Z_i(\ut=1) \right] + \phi(\xtm, \ut=\t) -  \phi(\xtm, \ut=1)  \nonumber \\
     &+ \log p_\x(\xtm | \ut=1) + \log | \det \m J \tilde{\ve g}(\xt, \xtm) | + \alpha_\t \label{eq:tcl_pxxu2}
\end{align}
Setting Eq.~\ref{eq:tcl_pxxu2} and Eq.~\ref{eq:tcl_pxxu} with $\ut = \t$ to be equal for arbitrary $\t$, we have:
 \begin{align}
     &\sum_{i=1}^{n} \sum_{j=1}^{k} \left(w_{ij\t} - w_{ij1}\right) \psi_{ij}(h_i(\xt, \xtm)) + \phi(\xtm, \ut=\t) - \phi(\xtm, \ut=1)+ \alpha_\t \nonumber \\
     &=\sum_{i=1}^{n} \sum_{j=1}^{k} \left(\lambda_{ij}(\ut=\t) - \lambda_{ij}(\ut=1)\right) q_{ij} (g_i(\xt, \xtm)) + \log p_\x(\xtm | \ut=\t) - \log p_\x(\xtm | \ut=1) + z_\t \label{eq:tcl_eq}
\end{align}
where $z_\t = \sum_{i=1}^{n} Z_i(\ut=1) - Z_i(\ut=\t)$.
By collecting this equation for all the $T$ labels into rows, except $\t=1$, which makes both-sides zero;
\begin{align}
	\m W^T \ve\psi(\ve h(\xt, \xtm)) + \ve\phi(\xtm) + \ve\alpha
	= \m L^T \ve q(\st) +  \ve p(\xtm) + \ve z, \label{eq:tcl_eq_vec}
\end{align}
where $\m W \in \R^{nk \times (T-1)}$ is a matrix of $w_{ij\t} - w_{ij1}$, with the product of $i, j$ giving row index and $\t$ column index,
$\m L$ is a matrix of $\lambda_{ij}(\ut=\t) - \lambda_{ij}(\ut = 1)$ given in the assumption~\ref{TA4} in the Theorem,
$\ve\psi(\ve h(\xt, \xtm)) = (\psi_{11}(h_{1}(\xt, \xtm)), \ldots, \psi_{nk}(h_{n}(\xt, \xtm)))^T$,
$\ve q(\st) = (q_{11}(s_1(t)), \ldots, q_{nk}(s_n(t)))^T$,
$\ve\phi(\xtm) =  (\phi(\xtm, \ut=2), \ldots, \phi(\xtm, \ut=T))^T - \ve 1 \phi(\xtm, \ut=1)$,
$\ve 1$ is a $(T-1) \times 1$ vector of ones,
$\ve\alpha = (\alpha_2, \ldots, \alpha_T)^T$,
$\ve p(\xtm) = (\log p_\x(\xtm | \ut=2), \ldots, \log p_\x(\xtm | \ut=T))^T - \ve 1 \log p_\x(\xtm | \ut=1)$,
and $\ve z = (z_2, \ldots, z_T)^T$.
Let a compound demixing-mixing function $\ve v(\st, \xtm) = \h \circ \tilde{\f}(\st, \xtm)$, 
and change variables to $\y = [\y_1, \y_2] = [\st, \xtm]$, we then have
\begin{align}
	\m W^T \ve\psi(\ve v(\y)) + \ve\phi(\y_2) + \ve\alpha
	= \m L^T \ve q(\y_1) +  \ve p(\y_2) + \ve z. \label{eq:tcl_eq_vec_v}
\end{align}

Firstly, we will show that $\m W$ has full row rank $nk$.
From the definition of $\ve q(\y_1)$, its partial derivative with respect to $y_{1i}$
is $\ve q'(y_{1i}) = (0, \ldots, 0, q_{i1}'(y_{1i}), \ldots ,q_{ik}'(y_{1i}), 0, \ldots, 0)^T$.
According to Lemma~3 of \citet{Khemakhem2020}, for $y_{1i}$ which satisfies A\ref{A2},
there exist $k$ points $(\bar{y}_{1i}^1, \ldots, \bar{y}_{1i}^k)$ 
such that $(\ve q'(\bar{y}_{1i}^1), \ldots, \ve q'(\bar{y}_{1i}^k))$ are linearly independent.
By differentiating Eq.~\ref{eq:tcl_eq_vec_v} with respect to $y_{1i}$ 
and collecting their evaluations at such $k$ distinctive points for all $i$, we get
\begin{align}
	\m W^T \tilde{\m Q} = \m L^T \m Q,
\end{align}
where $\m Q  \in \R^{nk \times nk}$ is a matrix collecting $\ve q'(\bar{y}_{1i}^l)$ to the columns indexed by $(i, l)$,
and $\tilde{\m Q}$ is a collection of partial derivatives of $\ve\psi(\ve v(\y))$ evaluated at the same points.
$\m Q$ is invertible (through a combination of Lemma~3 of \citet{Khemakhem2020} and the fact that each component of $\ve q$ is univariate), 
and thus the right-hand side has full column rank $nk$ because $\m L$ has full row rank $nk$ (assumption~\ref{TA4}).
The full column rank of the right-hand side implies the full row rank of $\m W$ and the invertibility of $\tilde{\m Q}$.

Now, let an augmented compound demixing-mixing function $\tilde{\ve v}(\y) = [\tilde{\ve v}_1(\y), \tilde{\ve v}_2(\y)] = \tilde{\h} \circ \tilde{\f}(\y)$,
where $\tilde{\h}$ is the augmented function defined in the assumption~\ref{TA6} in the Theorem.
The $\tilde{\ve v}_1(\y)$ corresponds to $\ve v(\y)$ defined above.
Note that $\tilde{\ve v}$ is invertible because both $\tilde{\h}$ and $\tilde{\f}$ are invertible.
What we need to prove is that $\tilde{\ve v}$ is a block-wise invertible point-wise function,
in the sense that $\tilde{v}_{1i}$ is a function of only one $y_{1j_i}$ and not of any of $y_{2j_i}$, and vise versa.
This can be done by showing that the product of any two distinct partial derivatives
of any component is always zero, and the Jacobian $\m J_{\tilde{\ve v}} \in \R^{2n \times 2n}$ is block diagonal;
the upper and lower block correspond to $\y_1$ and $\y_2$ respectively.
Along with invertibility, this means that each component depends exactly on one variable of the corresponding block ($\y_1$ or $\y_2$).
Below, we show that separately for $\m J_{\ve v} \in \R^{n \times 2n}$ and $\m J_{\tilde{\ve v}_2} \in \R^{n \times 2n}$.
Firstly, this is obviously true for $\m J_{\tilde{\ve v}_2}$ because $\tilde{\ve v}_2(\y)$ is just an identity mapping of
$\y_2$ from the definitions of $\tilde{\h}$ and $\tilde{\f}$, and does not depend on $\y_1$;
the lower non-zero block of $\m J_{\tilde{\ve v}}$ is an identity matrix.
Next, we will show that for $\m J_{\ve v}$.
We differentiate Eq.~\ref{eq:tcl_eq_vec_v} with respect to $y_{c}, 1 \le c \le n$ (an element of $\y_1=\st$), 
and $y_{d}, c < d \le 2n$, and get
\begin{align}
	\m W^T \frac{\partial^2}{\partial y_c \partial y_d} \ve\psi(\ve v(\y)) = 0.
\end{align}
From the full row rank of $\m W$ and the calculation of differentials, we get
\begin{align}
	\frac{\partial^2}{\partial y_c \partial y_d} \ve\psi(\ve v(\y)) = \m\Psi(\y)^T \ve\upsilon(\y) = 0,
\end{align}
where $\m\Psi(\y) = (\ve e^{(1,1)}(y_1), \ldots, \ve e^{(1,k)}(y_1), \ldots, \ve e^{(n,1)}(y_n), \ldots, \ve e^{(n,k)}(y_n)) \in \R^{2n \times nk}$,
$\ve e^{(a,b)} = (0, \ldots, 0, \psi_{ab}'(v_a), \psi_{ab}''(v_a), 0, \ldots, 0)^T \in \R^{2n}$,
such that the non-zero entries are at indices $(2a-1, 2a)$, 
$\ve\upsilon(\y) = (v_{1}^{c,d}(\y), v_{1}^{c}(\y)v_{1}^{d}(\y), \ldots, v_{n}^{c,d}(\y), v_{n}^{c}(\y)v_{n}^{d}(\y))^T \in \R^{2n}$,
$v_i^c = \frac{\partial v_i}{\partial y_c}(\y)$, and $v_i^{c,d} = \frac{\partial^2 v_i}{\partial y_c \partial y_d}(\y)$.
From Lemma~4 and 5 of \citet{Khemakhem2020}, assumption~\ref{TA7} implies that 
$\m\Psi(\y)$ has full row rank $2n$, and thus the pseudo-inverse of $\m\Psi(\y)^T$ fulfils $\m\Psi(\y)^{+T} \m\Psi(\y)^{T} = \m I$.
We multiply the equation above from the left by such pseudo-inverse and obtain
\begin{align}
	\ve\upsilon(\y) = 0.
\end{align}
In particular, $v_{a}^{c}(\y)v_{a}^{d}(\y) = 0$ for all $1 \le a \le n$, $1 \le c \le n$, and $c < d \le 2n$.
This means that a row of $\m J_\ve v \in \R^{n \times 2n}$ at each $\y$ has either
1)~only one non-zero entry somewhere in the former half block (corresponding to the partial derivatives by $\y_1$) 
or 2)~non-zero entries only in the latter half block (corresponding to the partial derivatives by $\y_2$).
The latter case is contradictory because it means that the component $v_i$ is a function of only $\y_2 = \xtm$,
and cannot hold Eq~\ref{eq:tcl_eq_vec_v}, which right-hand side is a function of all components of $\y_1$ (and $\y_2$).
Therefore, $\m J_\ve v$ should have only one non-zero entry in the former half block for each row.
From the results of $\m J_\ve v$ and $\m J_{\tilde{\ve v}_2}$,
we deduce that $\m J_{\tilde{\ve v}}$ is a block diagonal matrix.
Now, by invertibility and continuity of $\m J_{\tilde{\ve v}}$, we deduce that the location of the non-zero entries are fixed
and do not change as a function of $\y$. This proves that $\tilde{\ve v} = \tilde{\h} \circ \tilde{\f}(\y)$ is a block-wise invertible point-wise function,
and $v_{i}$ ($=h_i(\xt, \xtm)$) is represented by only one $y_{1j_i}$ ($=s_{j_i}(t)$) up to a scalar (component-specific) invertible transformation,
and the Theorem is proven.

\subsection{Discussion on the identifiability of IIA-HMM}
\label{sec:app_hmm}

We obtain the following Theorem on identifiability of IIA-HMM.
\begin{Theorem}
Assume the following:\vspace*{-2mm}
\begin{enumerate} 
\item We obtain observations from an NVAR model (Eq.~\ref{eq:f}), whose augmented model (Eq.~\ref{eq:fa}) is invertible and sufficiently smooth.\label{HA1}
\item The latent innovations of the process follow the assumption A\ref{A2} with $k \ge 2$, and the sufficient statistics $q_{ij}$ are twice differentiable.\label{HA2}
\item The $\ve u$ are unobserved (in contrast to the previous frameworks), and follow A\ref{A3}, where the transition matrix $\m A$ has full rank with non-zero diagonal entries, and induces irreducible stationary Markov chain with a unique stationary state distribution $\ve \pi$.\label{HA3}
\item The conditional distributions $p(\cdot | \xtm, \ut), \ut = 1, \ldots, C$ are all generically distinct for any $\xtm$, meaning that the set of points for which this doesn't hold is measure zero. \label{HA4}
\item The modulation matrix of size $nk \times (C-1)$\label{HA5}
	\begin{align}
		\m L = (\ve\lambda(2) - \ve\lambda(1), \ldots, \ve\lambda(T) - \ve\lambda(1))
	\end{align}
	has full row rank $nk$, where $\ve\lambda(c) = (\lambda_{11}(\ve u=c), \ldots, \lambda_{nk}(\ve u=c))^T \in \R^{nk}$.
\item We estimate the transition matrix, parameters of the innovation model, latent state at each data point, and demixing model $\h(\xt, \xtm): \R^{2n} \rightarrow \R^{n}$
	with universal approximation capability, by maximizing the likelihood of the observations.\label{HA6}
\item The augmented function $\tilde{\h}(\xt, \xtm) = [\h(\xt, \xtm), \xtm]: \R^{2n} \rightarrow \R^{2n}$ is invertible.\label{HA7}
\end{enumerate}\vspace*{-2mm}
Then, in the limit of infinite data, $\h$ provides a consistent estimator of the IIA model:
The functions $h_i(\xt, \xtm)$ give the independent innovations, up to permutation and scalar (component-wise) invertible transformations.
\end{Theorem}

\begin{proof}
	Assume equality of joint-data distributions for $2T+1$ observations from the IIA-HMM model with two different sets of parameters $\ve\theta, \hat{\ve\theta}$:
	$$p(\x_1, \dots, \x_T,\dots,\x_{2T+1}; \ve\theta)=p(\x_1, \dots, \x_T,\dots,\x_{2T+1}; \hat{\ve\theta})$$
	With the Assumptions~\ref{HA1}, \ref{HA3}, and \ref{HA4}, we can apply Lemma~\ref{lmhh:main_lemma} below and identify the following: 
	\begin{align}
		\ve A &= \hat{\ve A} \\
		\ve \pi &= \hat{\ve \pi} \\
		p(\x_t|\ut=c;\ve\theta)&=p(\x_t|\ut=\sigma(c);\hat{\ve\theta})\label{eqhh:marginal} \\
		p(\x_1,\dots,\x_T|\x_{T+1}, \ve u_{T+1}=c; \ve\theta)&=p(\x_1,\dots,\x_T|\x_{T+1}, \ve u_{T+1}=\sigma(c); \hat{\ve\theta}) \label{eqhh:H}.
	\end{align}
	where $\sigma(\cdot)$ accounts for permutation of labels. For rest of the proof, without loss of generality, we assume label ordering matches.
	Since joint distributions identify their marginals uniquely, equation \eqref{eqhh:H} implies
	\begin{align}
		p(\x_T| \x_{T+1}, \ve u_{T+1};\ve\theta) &=p(\x_T|\x_{T+1}, \ve u_{T+1}; \hat{\ve\theta}) \\
		\implies \frac{p(\x_T, \x_{T+1}| \ve u_{T+1};\ve\theta)}{p(\x_{T+1}| \ve u_{T+1};\ve\theta)} &=\frac{p(\x_T, \x_{T+1}| \ve u_{T+1};\hat{\ve\theta})}{p(\x_{T+1}| \ve u_{T+1};\hat{\ve\theta})}.
	\end{align}
	This, with \eqref{eqhh:marginal}, implies that the following is identified:
	\begin{align}
		p(\x_T, \x_{T+1}| \ve u_{T+1};\ve\theta)=p(\x_T, \x_{T+1}| \ve u_{T+1};\hat{\ve\theta}) \label{eqhh:main_identif}
	\end{align}
	Finally, notice that 
	\begin{align}
		\sum_{k=1}^C p(\x_T|\ve u_{T}=k; \ve \theta)p(\ve u_{T}=k|\ve u_{T+1}; \ve A)&=\sum_{k=1}^C p(\x_T|\ve u_{T}=k; \hat{ \ve\theta})p(\ve u_{T}=k|\ve u_{T+1};\hat{\ve A}) \nonumber \\
		\implies p(\x_T| \ve u_{T+1};\ve\theta)&=p(\x_T,| \ve u_{T+1};\hat{\ve\theta}) \label{eqhh:identif_second}
	\end{align}
	Writing out the log-likelihoods in equation \eqref{eqhh:main_identif}, we get:
	\begin{align}
		&\log p_{\tilde{\s}}(\tilde{\ve g}(\x_{T}, \x_{T+1}) | \ve u_{T+1}) + \log | \det \m J \tilde{\ve g}(\x_{T}, \x_{T+1})| =  \log \hat{p}_{\tilde{\s}}(\hat{\tilde{\ve g}}(\x_{T}, \x_{T+1}) | \ve u_{T+1}) + \log | \det \m J \hat{\tilde{\ve g}}(\x_{T}, \x_{T+1})| \nonumber \\
		\implies & \log p_\s(\g(\x_{T}, \x_{T+1}) | \ve u_{T+1}) + \log p_\x(\x_{T} | \ve u_{T+1}) + \log | \det \m J \tilde{\ve g}(\x_{T}, \x_{T+1})| \nonumber \\
		&= \log \hat{p}_\s(\hat{\g}(\x_{T}, \x_{T+1}) | \ve u_{T+1}) + \log \hat{p}_\x(\x_T | \ve u_{T+1}) + \log | \det \m J \hat{\tilde{\ve g}}(\x_{T}, \x_{T+1})|, \nonumber
	\end{align}
	where $p_{\tilde{\s}}$, $p_\s$, and $p_\x$ are the conditional pdfs of $(\s, \x)$, $\s$, and $\x$, respectively, and $\m J$ denotes the Jacobian.
	Using the result in \eqref{eqhh:identif_second}, gives us
	\begin{align}
		&\log p_\s(\g(\x_{T}, \x_{T+1}) | \ve u_{T+1}) + \log | \det \m J \tilde{\ve g}(\x_{T}, \x_{T+1})| = 	\log \hat{p}_\s(\hat{\g}(\x_{T}, \x_{T+1}) | \ve u_{T+1}) +  \log | \det \m J \hat{\tilde{\ve g}}(\x_{T}, \x_{T+1})|. \nonumber
	\end{align}
	Remainder of the proof follows as in \citet{Halva2020} and is not shown here for brevity. The general idea is to take the above equation for different values of $\ve u_{T+1}$ and use one of them as a `pivot' in order to get rid of the Jacobians. Finally, the exponential family distribution properties, as done in the earlier proofs of this paper, are used to show identifiability.
\end{proof}

\subsubsection{Lemmas}

\textbf{Set-up}: These Lemmas follow, in general, those of \citet{Alexandrovich2016supp} but with substantial modifications made to accomodate our model. We first define some relevant notation. Let $(X_t)_{t\in \mathbb{N}}$ denote the observed process and $(U_t)_{t\in\mathbb{N}}$ the discrete latent first-order Markov chain. Assume these processes are time-homogeneous. $K$ is the cardinality of the state-space of $U_t$, that is, the number of latent states. The first-order Markov chain for $U_t$ is governed by transition matrix $\ve A=(\alpha_{j,k})_{j,k=1,\dots,K}.$ Define $\set S \subset \R^q$ as any subset of Euclidean space. Suppose that $X_t$ takes values in $\set S$, and its distribution depends on its most recent past $X_{t-1}$ and the current latent state $U_t$ -- this distribution function is denoted by $F_{U_t,X_{t-1}}(X_t)$ and is time-homogeneous. Notice that subsequently, $X_t$ is independent of all $X_{t+s}$ for $|s| \ge 2$ given $X_{t-1}, X_{t+1}$ and $U_t$. $F_{U_t}(X_t)$ is used to denote the conditonal distribution of $X_t$ on $U_t$ alone; that is, all other variables have been integrated out. 
$\pi = (\pi_1, \ldots, \pi_K)$ denotes a stationary distribution of $\m A$. In the following proofs, $\mathbf{x}_t$ is not a random variable, but represents a point in $\set S$.

Let $\dim(V)$ denote the dimension of vector space $V$. For $\ve v_1, \dots, \ve v_n \in V$ let $\vspan \{\ve v_1,\dots,\ve v_n\}$ denote the subspace of $V$ spanned by $\ve v_1,\dots,\ve v_n$. For scalars $x_1, \dots, x_n \in \R$ let $\diag(x_1,\dots, x_n)$ denote $n$-dimensional diagonal matrix with $x_1, \dots, x_n$ along the diagonal. $\mathbf{1}_K$ is a $K$-dimensional vector of ones. Let $\m M_i \in \R^{K \times n_i} \, (n_i \in \mathbb{N}; i=1, 2, 3)$, then $[\m M_i, \m M_j]$ denotes the $K \times (n_i + n_j)$ matrix made by joining the two matrices at their columns. $(\m M)_{m,n}$ denotes the element of matrix $\m M$ on the $m$-th row and $n$-th column. Finally, let's define three-way arrays, indexed by $(i_1, i_2, i_3)$, where the corresponding element is given by:
\begin{equation}\label{3way}
	\langle \m M_1, \m M_2, \m M_3 \rangle_{(i_1, i_2, i_3)} = \sum_{k=1}^K (\m M_1)_{k, i_1} (\m M_2)_{k, i_2} (\m M_3)_{k, i_3} \, \,\,(i_j=1,\dots,n_j)
\end{equation}
Define kruskal rank $R_{\kappa}(\m M)$ as the maximal $j$ such that any selection of $j$ rows of $\m M$ are linearly independent. Theorem 4a of \citet{kruskal77supp} states that if:
\begin{equation}\label{eq:kruskal}
	R_{\kappa}(\m M_1) + R_{\kappa}(\m M_2) + R_{\kappa}(\m M_3) \ge 2K+2
\end{equation}
and $$ \langle \m M_1, \m M_2, \m M_3 \rangle = \langle \m N_1, \m N_2, \m N_3 \rangle $$
then there exist permutation matrix $\m P$ and diagonal matrices $\ve \Lambda_i$, such that $\ve \Lambda_1 \ve \Lambda_2 \ve \Lambda_3 = I_K$ and $\m N_i = \m \Lambda_i \m P \m M_i$.

\begin{Lemma} \label{lmhh:main_lemma}
	Assume that:
	\begin{enumerate}
		\item The latent state transition matrix $\m A$ has full rank and is ergodic.
		\item The conditional emission distributions $F_{u_t, \x_{t-1}}(\x_t)$ for $k=1,\dots,K$ are \textit{generically} distinct for each given $\x_{t-1}$. That is, the set of points $(\x_{t-1}, \x_t)$ for which this doesn't hold is measure zero. \label{as:generic}
		\item First-order Markov chain $(U_t)$ is stationary with starting distribution $\pi$, which is thus the stationary distribution of $\m A$
	\end{enumerate}
	then the marginal emission distributions $F_{u_t}(\x_t)$, the transition matrix $\m A$, the initial state probabilities $\pi$ are all identified from the joint distribution of the observation process $(X_1,\dots,X_{2T+1})$ where $T \ge K-1$, up to label swapping.
\end{Lemma}

\begin{proof}
\textit{\textbf{Step 1} (factorizing likelihood into blocks by conditional independence)}:
Consider we have $2T+1$ observations from our model. The likelihood of the model can be factored by conditioning on the variables at the central time point $T+1$ as per below:
\begin{align}\label{likelihood}
	\Pr(X_{1:2T+1} \le \x_{1:2T+1})&= \sum_{k} \Pr(X_{1:T} \le \x_{1:T}| U_{T+1}=k, X_{T+1} \le \x_{T+1}) \pi_{k} \\ \nonumber
		      &\times \Pr(X_{T+1} \le \x_{T+1}|U_{T+1}=k) \Pr(X_{T+2:2T+1} \le \x_{T+2:2T+1}|U_{T+1}=k, X_{T+1} \le \x_{T+1}),
\end{align}
where notation $\x_{1:2T+1} = (\x_1,\dots,\x_{2T+1})$ is used, and $\pi_{k}$ represents the stationary distribution $\Pr(U_{T+1} = k)$. Assume $T \ge K-1$. The factorial structure of the likelihood allows us to consider two random variables 
$$V_T = X_{1:T} = (X_1,\dots,X_T) \quad \text{and} \quad W_T=X_{T+2:2T+1}=(X_{T+2},\dots,X_{2T+1}).$$ 
The conditional distribution of $W_T$, evaluated at some $\y_{1:T} \in \set S^T$, given $X_{T+1}=\y_0 \in \set S$ and $U_{T+1}=k$ can be written as:
$$G_T(\y_{0:T};k)=\Pr(W_T \le \y_{1:T}|U_{T+1}=k, X_{T+1} \le \y_0)=\sum_{k_1\cdots k_T}\alpha_{k,k_1}\prod_{t=2}^T\alpha_{k_{t-1},k_t}\prod_{t=1}^T F_{k_t, \y_{t-1}}(\y_t).$$ 
For the conditional likelihood of $V_T$ on $U_{T+1}$ and $X_{T+1}$, we must consider time reversal:
\begin{align*}
	&\tilde{\m A}=(\tilde{\alpha}_{j,k})_{j,k=1,\dots,K}, \, \, \, \, \tilde{\alpha}_{j,k}=\frac{\pi_k \alpha_{k,j}}{\pi_j},\\
	&\tilde{F}_{k, \x}(\x_t)=\Pr(X_t \le \x_t|U_t=k, X_{t+1} \le \x).
\end{align*}
Then for $\y_{T:1}=(\y_{T},\dots,\y_1) \in \set S^{T}$, and given $X_{T+1}=\y_{0} \in \set S$ and $U_{T+1}=k$ can be written as:
\begin{align*}
	H_T(\y_{T:0};k)&=\Pr(V_T \le \y_{T:1}|U_{T+1}=k, X_{T+1} \le \y_{0})\\
			 &=\sum_{k_1\cdots k_T}\tilde{\alpha}_{k,k_1}\prod_{t=2}^{T}\tilde{\alpha}_{k_{t-1},k_{t}}\prod_{t=1}^T\tilde{F}_{k_t, \y_{t-1}}(\y_t).
\end{align*}
Now, take any arbitrary points $\bar{\x} \in \set S$ and $\z_j,\tilde{\z}_j \in \set S^T\text{ for } j=1,\dots,K$. Define $\z_j^+=(\bar{\x},\z_j)$ and $\tilde{\z}_j^+ = (\bar{\x}, \tilde{\z}_j)$. The likelihood in \eqref{likelihood}, at these arbitrary points, for some $j$, can thus be formulated as:
\begin{align}\label{likelihood2}
	\Pr(X_{1:2T+1} \le (\tilde{\z}_j,\bar{\x},\z_j))= \sum_{k} H_T(\tilde{\z}_j^+;k) \pi_{k} F_{k} (\bar{\x})G_T(\z_j^+;k).
\end{align}
Note the correspondence of the above equation to \eqref{3way}. Now consider $K \times K$ matrix:
\begin{equation} \label{G1}
	\m G_1 = (G_T(\z_j^+; k))_{k,j=1,\dots,K}=(G_T((\bar{\x},\z_j); k))_{k,j=1,\dots,K}.
\end{equation}
From Lemma \ref{lm2} below we have that there exist $\z_1,\dots,\z_K \in \set S^T$ such that $\m G_1$ is full rank, for any $\bar{\x}$. Similarly, we form matrix:
\begin{equation} \label{H1}
	\m H_1 = (H_T(\tilde{\z}_j^+; k))_{k,j=1,\dots,K}=(H_T((\bar{\x},\tilde{\z}_j); k))_{k,j=1,\dots,K}.
\end{equation}
Again, by Lemma \ref{lm2}, there exists $\tilde{\z}_1,\dots,\tilde{\z}_K \in \set S^T$ for which $\m H_1$ has full rank.\\

\textit{\textbf{Step 2} (Identifying the distribution as a three-way array):} Now, let $\ve v, \tilde{\ve v} \in \set S^T$ be \textit{any} arbitrary points (\textit{c.f.} \eqref{G1},\eqref{H1} focused on the existence of \textit{some} points). Let $\bar{\x} \in \set S$ also be any point but such that it matches always the one in $\m G_1$ and $\m H_1$. From Assumption~2, we have that $K\times 2$ matrix
\begin{equation}
	\m M_2 = [(F_i(\bar{\x}))_{i=1,\dots,K}, \mathbf{1}_K]
\end{equation}
has Kruskal rank of $2$. From Step 1, the $K \times (K+2)$-matrices
\begin{align}
	&\m M_3 = [\m G_1, (G_T((\bar{\x}, \ve v); k))_{k=1,\dots,K}, \mathbf{1}_K], \quad \m M_1 = [\m H_1, (H_T((\bar{\x}, \tilde{\ve v}); k))_{k=1,\dots,K}, \mathbf{1}_K] \nonumber \\
	&\tilde{\m M}_1 = \diag(\pi)\m M_1,
\end{align}
all have full ranks, $K$, where $\pi_k >0 \,\,\, (k=1,\dots,K)$, and hence the Kruskal rank condition
\begin{equation}
	R_{\kappa}(\tilde{\m M}_1) + R_{\kappa}(\m M_2) + R_{\kappa}(\m M_3) \ge 2K+2
\end{equation}
is satisfied for the three-way array 
\begin{equation}
	\m M^{\star} = \langle \tilde{\m M}_1, \m M_2, \m M_3 \rangle.
\end{equation}
The question now is whether the distribution of $(X_1,\dots,X_{2T+1})$ is alone sufficient to identify $\m M^{\star}$. To see that this is the case, consider the following, exhaustive, possibilities:
\begin{align*}
	&\text{For $i < K+2$; $j=1$; $r < K+2$} \nonumber \\
	& \quad \quad \m M_{(i, 1, r)}^{\star} = \sum_{k=1}^K \pi_k H_T((\bar{\x}, \tilde{\ve v}_i); k) F_k(\bar{\x}) G_T((\bar{\x}, \ve v_r); k) = \Pr(X_{1:T} \le \tilde{\ve v}_i, X_{T+1} \le \bar{\x}, X_{T+2:2T+1} \le \ve v_r) \\
        &\text{For $i=K+2$; $j=1$; $r < K+2$}& \nonumber \\
	& \quad \quad \m M_{(K+2, 1, r)}^{\star} = \sum_{k=1}^K \pi_k F_k(\bar{\x}) G_T((\bar{\x}, \ve v_r); k) = \Pr(X_{T+1} \le \bar{\x}, X_{T+2:2T+1} \le \ve v_r) \\
        &\text{For $i<K+2$; $j=1$; $r = K+2$}& \nonumber \\
	& \quad \quad \m M_{(i, 1, K+2)}^{\star} = \sum_{k=1}^K \pi_k H_T((\bar{\x}, \tilde{\ve v}_i); k) F_k(\bar{\x}) = \Pr(X_{1:T} \le \tilde{\ve v}_i, X_{T+1} \le \bar{\x}) \\
	&\text{For $i=K+2$; $j=1$; $r = K+2$}& \nonumber \\
	& \quad \quad \m M_{(K+2, 1, K+2)}^{\star} = \sum_{k=1}^K \pi_k  F_k(\bar{\x}) = \Pr(X_{T+1} \le \bar{\x}). \\
	&\text{For $i < K+2$; $j=2$; $r < K+2$} \nonumber \\
	& \quad \quad \m M_{(i, 2, r)}^{\star} \,\begin{aligned}[t] &= \sum_{k=1}^K \pi_k H_T((\bar{\x}, \tilde{\ve v}_i); k) G_T((\bar{\x}, \ve v_r); k) = \Pr(X_{1:T} \le \tilde{\ve v}_i, X_{T+2:2T+1} \le \ve v_r| X_{T+1} \le \bar{\x}) \\   &= \frac{\Pr(X_{1:T} \le \tilde{\ve v}_i,  X_{T+1} \le \bar{\x}, X_{T+2:2T+1} \le \ve v_r)}{\Pr(X_{T+1} \le \bar{\x})} \end{aligned}\\
	&\text{For $i = K+2$; $j=2$; $r < K+2$} \nonumber \\
	& \quad \quad \m M_{(K+2, 2, r)}^{\star} = \sum_{k=1}^K \pi_k G_T((\bar{\x}, \ve v_r); k) = \Pr(X_{T+2:2T+1} \le \ve v_r| X_{T+1} \le \bar{\x}) \\ 
	&\text{For $i<K+2$; $j=2$; $r = K+2$}& \nonumber \\
	& \quad \quad \m M_{(i, 2, K+2)}^{\star} = \sum_{k=1}^K \pi_k H_T((\bar{\x}, \tilde{\ve v}_i); k)  = \Pr(X_{1:T} \le \tilde{\ve v}_i | X_{T+1} \le \bar{\x}) \\
	&\text{For $i=K+2$; $j=2$; $r = K+2$}& \nonumber \\
	& \quad \quad \m M_{(K+2, 2, K+2)}^{\star} = 1
\end{align*}
These are all uniquely determined by the joint distribution of $(X_1,\dots,X_{2T+1})$ (joint distribution uniquely defines marginals).

\textit{\textbf{Step 3} (identifying parameters from three-way arrays):} Next assume we have an alternate set of parameters to those above; transition matrix $\widehat{\m A}$, arbitrary initial state distribution $\hat{\pi}$ (not necessarily stationary), and distribution function $\widehat{F}_{u,\x}$ defined analogously to above. These parameters define matrices $\m N_i (i=1,2,3)$, which are defined, and evaluated at the same points, as $\m M_i$ from above. Further, $\tilde{\m N}_1=\diag(\hat{\pi}\widehat{\m A}^T)\m N_1$, where $\hat{\pi}\widehat{\m A}^T$ is the marginal distribution of $U_{T+1}$. If the two sets of parameters induce the same joint distribution $(X_1, \dots, X_{2T+1})$ then Step 2 ensures that
$$ \langle \tilde{\m M}_1, \m M_2, \m M_3 \rangle = \langle \tilde{\m N}_1, \m N_2, \m N_3 \rangle $$
And, due to Theorem 4a \citet{kruskal77supp}, since $\tilde{\m M}_1$, $\m M_2$, $\m M_3$ satisfy \eqref{eq:kruskal}, there are $K\times K$ permutation matrix $\m P$ and  scaling matrices $\ve \Lambda_i, (i=1,2,3)$ with $\ve \Lambda_1 \ve \Lambda_2 \ve \Lambda_3=I_K$, such that
\begin{align}
	\m M_i = \ve \Lambda_i \m P \m N_i \,\, (i=2,3) \quad \text{ and } \quad \tilde{\m M}_1 = \ve \Lambda_1 \m P \tilde{\m N}_1.
\end{align}
Since $\m M_i, \m N_i \,\, (i=2,3)$ have only ones in the last column, $\ve \Lambda_2= \ve \Lambda_3 = I_K$ and thus also $\ve \Lambda_1=I_K$. The first consequence of this is that $H_T((\bar{\x}, \tilde{\ve v}); k)$, $F_u(\bar{\x})$, and $G_T((\bar{\x}, \ve v); k)$ are identified, up to simultaneous permutation of labels, for arbitrary $\ve v,\tilde{\ve v} \in \set S$ and given $\bar{\x}$. But notice that we can construct above argumentation for any $\bar{\x}$.

Further, as $\tilde{\m M}_1$ and $\m M_3$ are full rank, then so must be $\tilde{\m N}_1$ and $\m N_3$. This in turn means that $\m P$ is uniquely determined and $\pi = \hat{\pi}\widehat{\m A}^T$ as they are both in the last columns of $\tilde{\m M}_1=\tilde{\m N}_1$.

\textit{\textbf{Step 4} (identifying the transition matrix):} We show this for $T=K-1$. In \textit{Step 1}, we considered the matrix 
\begin{align}
	\m G_1 = (G_{T}((\x_0,\z_j); k))_{k,j=1,\dots,K}.
\end{align}
Now consider instead a one time-step longer sequence, with only the first observation different, keeping labeling fixed:
\begin{align}
	\m G = (G_{T+1}((\x, \x_0,\z_j); k))_{k,j=1,\dots,K}.
\end{align}
From \textit{Step 2}, $H_{T+1}(\cdot;k), F_k, G_{T+1}(\cdot;k)$ are identified up to joint label swapping and thus so is $\m G$. $\m G_1$ and $\m A$ are related by
$$\m G = \m A \m D_{\x}(\x_0)\m G_1,$$
where $\m D_{\x}(\x_0)=\diag(F_{1,\x}(\x_0),\dots,F_{K,\x}(\x_0))$,
and therefore
$$\m A = \m G \m G_1^{-1}\m D_{\x}(\x_0)^{-1}.$$
Thus $\m A $ can be identified from above (for a large enough $\x_0$ as to avoid issues in the inverse), as all the constituents are identified, so $\m A =\hat{\m A}$. Also, as $\m A$ is invertible and from above we can now get that $\pi = \hat{\pi}\m A^T$, which combined with $\pi\m A^{-1}=\pi$ gives $\pi=\hat{\pi}$.
\end{proof}

\begin{Lemma} \label{lm1} Let $t \le K-1$ and $\m B_1,\dots,\m B_t$ be full-rank matrices in $\R^{K\times K}$ such that $\m B_1 \mathbf{1}_K,\dots,\m B_t \mathbf{1}_K$ are linearly independent vectors. Let $\m A$ be a $K \times K$ full rank transition matrix, and $F_{1, \x_0}(\x),\dots, F_{K, \x_0}(\x)$ distribution functions satisfying Assumption~\ref{as:generic}. Then, for every $\x_0$, there exists some $\x^{\star} \in \set S$ and $j\in \{1,\dots, t\}$ for which the $K \times (t+1)$-matrix
	$$ \left[\m B_1 \m A \mathbf{1}_K,\dots,\m B_t \m A \mathbf{1}_K, \m B_{j} \m A \m D_{\x_0}(\x^{\star}) \mathbf{1}_K   \right] $$
	is full rank.
\end{Lemma}

\begin{proof}
Since $\m A$ is a proper transition matrix, we have that
	$$\m M = \left[\m B_1\mathbf{1}_K,\dots,\m B_t \mathbf{1}_K \right] = \left[ \m B_1\m A \mathbf{1}_K,\dots,\m B_t \m A \mathbf{1}_K \right],$$
	and therefore $\m B_1\m A \mathbf{1}_K,\dots,\m B_t \m A \mathbf{1}_K$ are linearly independent, with $S_1=\vspan\{ \m B_1 \m A \mathbf{1}_K,\dots,\m B_t \m A \mathbf{1}_K \}$ and $\dim(S_1)=t$. The Lemma can now be proven by contradiction. Assume that for any $j$, $\m B_j\m A \m D_{\x_0}(\x^{\star}) \mathbf{1}_K$ is in the span $S_1$. We can write $\m Q_j=\m B_j\m A$, and notice that this is full-rank for all $j$. Hence
	$$\m B_j \m A \m D_{\x_0}(\x^{\star}) \mathbf{1}_K=\m Q_j \m D_{\x_0}(\x^{\star}) \mathbf{1}_K=\sum_{i=1}^K F_{i,\x_0}(\x^{\star}) \mathbf{q}_{j,i},$$
	where $\mathbf{q}_{j,i}$ denotes the $i$-th column vector of $\m Q_j$, and we thus have a conic (positive) combination of $K$ linearly independent vectors. If we consider all feasible $\x^{\star}$, this defines a subspace of conical hull in $K$ dimensions. This contradicts the assumption of $\m B_j\m A \m D_{\x_0}(\x^{\star}) \mathbf{1}_K$ being in the span $S_1$ for all $\x^{\star}$ and thus concludes the proof.
\end{proof}

\begin{Lemma} \label{lm2} Under Assumption~2 (of Lemma~\ref{lmhh:main_lemma}), for $T \ge K-1$ the conditional distributions of $W_T$ given $U_{T+1}=k (k=1,\dots,K)$ and $X_{T+1}=\x_0 \in \set S$, that is the functions $G_T((\x_0, \cdot); k)$, are linearly independent over $k=1,\dots,K$ for any fixed $\x_0$, and furthermore, there exist $\z_1,\dots,\z_K \in \set S^T$ such that the matrix
\begin{equation*}
	\m G_1 = (G_T(\z_j^+; k))_{k,j=1,\dots,K}=(G_T(\x_0,\z_j; k))_{k,j=1,\dots,K},
\end{equation*}
has full rank K.\end{Lemma}

\begin{proof}
Recall:
\begin{equation}
	G_t((\x_0, \x_{1:t});k)=\sum_{k_1\cdots k_t}\alpha_{k,k_1}\prod_{s=2}^t\alpha_{k_{s-1},k_s}\prod_{s=1}^tF_{k_s, \x_{s-1}}(\x_s).
\end{equation}
Define $K \times K$ stochastic diagonal matrix $\m D_{\x_{t-1}}(\x_t)=\diag(F_{k_t=1,\x_{t-1}}(\x_t),\dots,F_{k_t=K,\x_{t-1}}(\x_t))$, and $\ve \alpha_l$ the $l$-th row vector of the transition matrix $\m A$. We can then write:
\begin{equation}
	G_t((\x_0, \x_{1:t});k)= \ve \alpha_k \m D_{\x_0}(\x_1)\m A \m D_{\x_1}(\x_2)\m A \cdots \m A \m D_{\x_{t-1}}(\x_t)\mathbf{1}_K.
\end{equation}
And define:
\begin{equation}
	\tilde{G}_t((\x_0, \x_{1:t});k)= F_{k, \x_0}(\x_1)\ve \alpha_k \m D_{\x_1}(\x_2)\m A \cdots\m A \m D_{\x_{t-1}}(\x_t)\mathbf{1}_K.
\end{equation}
It follows that $$\m G_1 = \m A(\tilde{G}_{K-1}((\x_0, \z_j); k)_{k,j=1,\dots,K}=\m A \tilde{\m G}_1,$$
and therefore it suffices to prove the lemma for $\tilde{\m G}_1$.

Proof by induction is used to show that there exist
\begin{align} \label{claima}
	\z_1^{(t)},\dots,\z_{t+1}^{(t)} \in \set S^t \,\,(t=1,\dots,K-1),
\end{align}
for which the vectors (i.e. columns of $\tilde{\m G}_1^{(t)}$)
\begin{align} \label{claimb}
	\ve v_j^{(t)} = \begin{bmatrix}\tilde{G}_t((\x_0, \z_j^{(t)}; 1) & \cdots & \tilde{G}_t((\x_0, \z_j^{(t)}; K)\end{bmatrix}' \, \,(j=1,\dots,t+1),
\end{align}
are linearly independent, and $\ve v_1^{(t)}$ has strictly positive entries. Note that the superscipts $(t)$ are used just to keep track of the $t$ being considered. The case $t=K-1$ will establish the lemma. In other words, we will only prove the theorem up to $T=K-1$. Since marginal distributions of linearly dependent distributions remain linearly dependent, linear independence follow for any $T \ge K-1$, and the existence of corresponding points $\z_1,\dots,\z_K \in \set S^T$ follows from Lemma 17 in \citet{Allman2009supp}. 

\textit{Proof by induction -- base case}: Set $t=1$. In this instance, $\tilde{\m G}_1^{(1)}$ is $K \times 2$, with columns given by:
$$\ve v_j^{(1)} = \begin{bmatrix} F_{1, \x_0}(\z_j^{(1)}),  & \cdots, & F_{K, \x_0}(\z_j^{(1)})\end{bmatrix}' \, \,(j=1,2).$$
By Assumption~2, the $K$ density functions with fixed $\x_0$ are distinct and therefore $\ve v_1^{(1)}$ and $\ve v_2^{(1)}$ are linearly independent with $\ve v_1^{(1)}$ strictly positive. 

\textit{Proof by induction -- induction step}: For induction, suppose that the claim \eqref{claima}-\eqref{claimb} holds for some $t < K-1$. Equation \eqref{claimb} can be rewritten as:
\begin{align}
	\ve v_j^{(t)} =  [ &F_{1, \x_{0,j}}(\x_{1,j})\ve \alpha_1 \m D_{\x_{1,j}}(\x_{2,j})\m A \cdots\m A \m D_{\x_{t-1, j}}(\x_{t, j})\mathbf{1}_K, \cdots, \nonumber \\
			&F_{K, \x_{0,j}}(\x_{1,j})\ve \alpha_K \m D_{\x_{1,j}}(\x_{2,j})\m A \cdots\m A \m D_{\x_{t-1, j}}(\x_{t, j})\mathbf{1}_K]' \nonumber \\
	=[& \underbrace{\m D_{\x_{0,j}}(\x_{1,j})\m A \m D_{\x_{1,j}}(\x_{2,j})\m A \cdots \m A \m D_{\x_{t-1, j}}(\x_{t, j})}_{\m B_j(\z_j^{(t)})}\mathbf{1}_K], \label{eq:v_mat_form}
\end{align}
and thus we have
\begin{align*}
	\tilde{\m G}_1^{(t)} =[\ve v_1^{(t)},\cdots,\ve v_{t+1}^{(t)}]= [\m B_1(\z_1^{(t)})\mathbf{1}_K, \cdots, \m B_{t+1}(\z_{t+1}^{(t)})\mathbf{1}_K].
\end{align*}
All $\m B_j (j=1,\dots,t+1)$ are full rank, and by the inductive assumption all the vectors are linearly independent. It follows from Lemma \ref{lm1} that there exists $j\in(1,\dots,t+1)$ and $\x^{\star}$ for which the $K\times(t+2)$ matrix: 
\begin{align}
	\m M =  \left[\m B_1(\z_1^{(t)}) \m A \mathbf{1}_K,\dots,\m B_{t+1}(\z_{t+1}^{(t)}) \m A \mathbf{1}_K, \m B_{j}(\z_j^{(t)}) \m A \m D_{\x_t}(\x^{\star}) \mathbf{1}_K   \right] 
\end{align}
has full rank $t+2$,and hence a $(t+2)\times(t+2)$ submatrix of non-zero determinant. Since $\m D_{\x_{t-1}}(\x_t) \rightarrow I$ when $\x_t \rightarrow \infty$,
\begin{align}
	\left[\m B_1(\z_1^{(t)}) \m A \m D_{\x_{t}}(\x) \mathbf{1}_K,\dots,\m B_{t+1}(\z_{t+1}^{(t)}) \m A \m D_{\x_{t}}(\x) \mathbf{1}_K, \m B_{j}(\z_j^{(t)}) \m A \m D_{\x_t}(\x^{\star}) \mathbf{1}_K   \right] \rightarrow \m M,\quad \x \rightarrow \infty 
\end{align}
and hence the corresponding submatrix will also have non-zero determinant in above, for an appropriate $\x \in \set S$.
Notice also how above defines $\ve v_j^{(t+1)}\,\, (j=1,\dots,t+2)$, as per equation \eqref{eq:v_mat_form}. Therefore the claim for $t+1$ is satisfied by setting 
\begin{align}
	\z_s^{(t+1)} = \left[\z_s^{(t)}, \x \right] \,\,(s=1,\dots,t+1) \quad \z_{t+2}^{(t+1)} = \left[\z_j^{(t)}, \x^{\star} \right] 
\end{align}
and so the proof concludes.

\end{proof}

\subsection{Implementation Detail for Simulation~1}
\label{sec:app_sim1}

We give here more detail on the data generation, training, and evaluation for IIA-GCL in Simulation~1 (Section~\ref{sec:sim1}).

\paragraph{Data Generation}
We generated data from an artificial NVAR process with time-index-parameterized nonstationary innovations.
The nonstationary innovations were randomly generated from a Gaussian distribution 
by modulating its mean and standard deviation across time $t$; 
i.e., the auxiliary variable $\ut = t$, and $\log p(s_i(t)) \propto - \lambda_{i1}(t) s_i(t)^2 - \lambda_{i1}(t)\lambda_{i2}(t) s_i(t)$, 
where $\lambda_{i1}(t)$ and $\lambda_{i2}(t)$ control the standard deviation and mean of the $i$-th component at time point $t$, respectively.
Each of $\lambda_{i1}(t)$ and $\lambda_{i2}(t)$ was modeled to be temporally smooth and continuous, by 1) obtaining a combination of
Fourier basis functions spanning the whole time series (sine and cosine bases with 64 frequencies),
which weights were randomly selected from uniform distribution, 2) normalizing to $[-2, 2]$, 
and 3) (only for $\lambda_{i1}(t)$) putting into exponential function.
The dimensions of the observations and innovations ($n$) were 20.
As the NVAR model, we used a multilayer perceptron we call NVAR-MLP,
which takes a concatenation of $\xtm$ and $\st$ as an input, then outputs $\xt$.
To guarantee the invertibility, we fixed the number of units of each layer to $n$,
and used leaky ReLU units for the nonlinearity except for the last layer which has no-nonlinearity.

\paragraph{Training}
Considering the innovation model with $\ut = t$, 
we here used IIA-GCL for the estimation of the latent innovations.
We adopted MLPs as the nonlinear scalar functions in Eq.~\ref{eq:gcl_r}.
The MLP for $\h$ ({\it h-MLP})
outputs $n$-dimensional feature values from an input $(\xt, \xtm)$,
which is supposed to represent the latent innovations after the training. 
The number of layers was selected to be the same as that of the NVAR-MLP,
and the number of node in each layer was $4n$ except for the output layer ($n$),
so as to make it have enough number of parameters as the demixing model.
A {\it maxout} unit was used as the activation function in the hidden layers,
which was constructed by taking the maximum across two affine fully connected weight groups,
while no-nonlinearity was applied at the last layer.
The scalar functions $\psi_{ij}$, $\mu_{ij}$, and $\alpha(\ut)$ were modeled to be consistent with the NVAR model; i.e.,
we incorporated the information into the model that 
1)~the innovations were generated based on the Gaussian distribution with mean and std modulations
by the log-pdf shown above, and 
2)~$\lambda_{i1}$ and $\lambda_{i2}$ were generated 
through a combination of Fourier basis functions with 64 frequencies, 
while their weights have to be estimated from the data.
For $\phi$, which has dependency on $\ut$, we used the same structure 
as the combination of $\h$, $\psi_{ij}$, and $\mu_{ij}$ explained above, which we call $\phi$-MLP,
except that the $\phi$-MLP takes a single data point ($\xtm$) as an input, instead of a set of the consecutive points $(\xt, \xtm)$.
The regression function also needs additional terms representing the marginal distributions of $\ve s$ and $\ve x$ ($\beta$ and $\gamma$),
which were here modeled by the weighted squared sum of the output units of the h-MLP and $\phi$-MLP, respectively.

The nonlinear regression function was trained by back-propagation with a momentum term
so as to discriminate the real dataset from its $\ut$-randomized version. 
The initial parameters were randomly drawn from a uniform distribution. 
The performance was evaluated by the Pearson correlation between the true innovations and the estimated feature values $\h$.
It was averaged over 10 runs, for each setting of the complexity (number of layers) $L \in [1,3,5]$ of the NVAR-MLP and the number of data points.

For comparison, we also applied NICA based on GCL (NICA-GCL; \citet{Hyvarinen2019}),
an NVAR with additive innovation model (AD-NVAR),
and variational autoencoder (VAE; \citet{Kingma2014}) to the same data. 
For all of them, we fixed the number of layers of the demixing model to be the same as that of the NVAR-MLP.
We fixed $L \in [1,2]$ exceptionally for VAE because of the instability of training in high layer models.
See Supplementary Material~\ref{sec:app_baseline} for the details of the baseline methods.

\subsection{Implementation Detail for Simulation~2}
\label{sec:app_sim2}

We give here more detail on the training for IIA-TCL in Simulation~2 (Section~\ref{sec:sim2}).

\paragraph{Training}
We applied IIA-TCL to the same data used in Simulation~1.
For IIA-TCL, we first divided the time series into 256 equally-sized segments, 
and used the segment label as the auxiliary variable $\ut$;
i.e., we assume that the data are segment-wise stationary,
Although this assumption is not consistent with the real innovation model (Simulation~1), it is approximately true because the modulations were temporally smooth and continuous; we thus consider here data with a realistic deviation from model assumptions.
We adopted MLPs as the nonlinear scalar functions in the regression function (Eq.~\ref{eq:tcl_r}).
The architecture of the MLP for $\h$ ({\it h-MLP}) was the same as that in Simulation~1.
Considering the log-pdf of the innovation,
we fixed $\psi_{i1}(y_i) = y_i^2$, and $\psi_{i2}(y_i) = y_i$.
For $\phi$, which has dependency on $\ut$, we used the same structure 
as the combination of $\h$, $\psi_{ij}$, and $w_{ij\t}$, 
except that $\phi$ takes a single data point ($\xtm$) as an input, instead of a set of consecutive points $(\xt, \xtm)$.
The training and evaluation methods follow those in Simulation~1.
We discarded the cases of small data sets ($2^{10}$ and $2^{12}$, corresponding to 4 and 16 samples in a segment)
because of the instability of training.

For comparison, we also applied NICA (TCL; \citet{Hyvarinen2016}).
See Supplementary Material~\ref{sec:app_baseline} for the details of the baseline methods.


\subsection{Simulation~2 in two-dimensional space}
\label{sec:app_sim2_2d}

We conducted an additional simulation to visually demonstrate the advantage of the IIA framework. 
The settings were the same to Simulation~\ref{sec:sim2} (see Supplementary Material~\ref{sec:app_sim2})
except that the dimensions of the observations and the innovations were two,
the number of layers was 5, and the number of data points was $2^{18}$.
The estimated innovations by IIA-TCL looks
clearly better demixed compared to the baseline methods (AD-NVAR and NICA-TCL; see Supplementary Material~\ref{sec:app_baseline} for the details).

\begin{figure*}[h]
 \centering
 \includegraphics[width=\columnwidth]{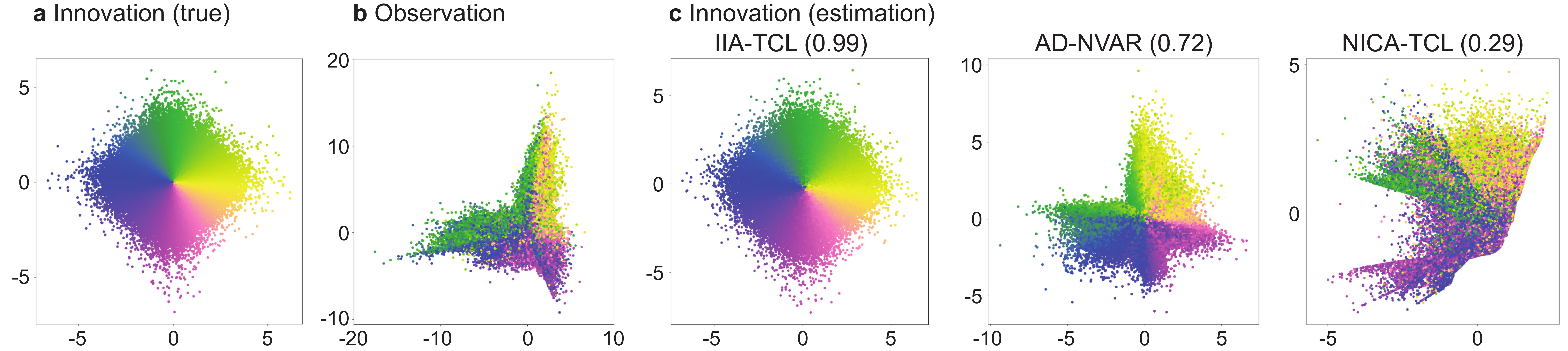}
 \caption{Estimation of the latent innovations from unknown artificial two dimensional NVAR process.
 (\textsf{\textbf{a}})~Scatter plot of the true innovations. 
 (\textsf{\textbf{b}})~Observations.
 (\textsf{\textbf{c}})~Innovations estimated by IIA-TCL, and for comparison, by AD-NVAR, and NICA-TCL. 
 The values show the mean absolute correlation coefficients between innovations and their estimates.
 }
 \label{fig:sim2d}
\end{figure*}

\subsection{Implementation Detail for Simulation~3}
\label{sec:app_sim3}

We give here more detail on the training for IIA-HMM in Simulation~3 (Section~\ref{sec:sim3}).

\paragraph{Data Generation}
We generated data from an artificial NVAR with hidden Markov chain.
The innovations were generated based on the method used in \citet{Halva2020}.
Briefly, the innovations were generated by Gaussian emission distributions of an HMM with $C$ discrete states, where
the means and the variances of the Gaussian distribution were selected to be distinctive across components/states.
The transition matrix was defined to have 99\% probability to stay at the current state,
and 1\% probability to switch to the next state, in cyclic manner.
We fixed the dimension of the innovations ($n$) to 5,
and the number of latent states was set to $C=2n + 1$. 
The observations were then obtained by randomly generated NVAR-MLP (see Supplementary Material~\ref{sec:app_sim1}), using the generated innovations.

\paragraph{Training}
We used here EM algorithm to maximize the likelihood for estimating the demixing model $\h$ based on MLP (h-MLP), 
the transition probability matrix $\m A$, the latent state at each data point, 
and the mean and the variance parameters of each state.
The implementation is based on that of NICA-HMM (\citet{Halva2020}; github.com/HHalva/hmnlica),
with some differences such as the demixing model and the incorporation of the margianl distribution $p(\ve x_0)$ (see Eq.~\ref{eq:hmm_likelihood}).
Although the likelihood includes the determinant of the Jacobian, which is widely considered difficult to compute,
we can numerically calculate its gradient by utilizing recent developments of the numerical calculation of gradients (here, JAX library).
The number of layers of h-MLP was selected to be the same as that of the NVAR-MLP,
and the number of node in each layer was $2n$ except for the output layer ($n$).
A smooth version of leaky ReLU was used as the activation function in the hidden layers; $
y = ax + (1-a) \log(1 + \exp^x)$, where $x$ is the input, $y$ is the output, and $a$ is the leak coefficient.
This type of differentiable function is useful for the stable estimation by the EM algorithm.
No-nonlinearity was applied at the last layer.
For better initialization of the h-MLP parameters than the random values, we firstly applied IIA-TCL to the observation
with assuming segment-wise stationarity (length of segments was 32), then used it as the initial values of the h-MLP.
Due to the sensitivity of the algorithm to the initial values of the parameters,
we repeated the estimation 20 times with different initializations, 
then selected the one with the highest likelihood.
The evaluation methods follow those in Simulation~1.
For comparison, we also applied NICA based on HMM (NICA-HMM; \citet{Halva2020}),
an NVAR with additive innovation model (AD-NVAR),
and IIA-TCL which was also used as the initialization.
For all of them, we fixed the number of layers of the demixing model to be the same as that of the NVAR-MLP.
See Supplementary Material~\ref{sec:app_baseline} for the details of the baseline methods.
%

\subsection{Detail for Experiments on Real Brain Imaging Data}
\label{sec:app_exp}

\paragraph{Data and Preprocessing}
We used a publicly available MEG dataset (\citet{Westner2018}; https://doi.org/10.17605/OSF.IO/M25N4). 
Briefly, the participants were presented with a random word selected from 420 unrelated German nouns 
(duration = $697 \pm 119$~ms) either visually (projected centrally on a screen) 
or auditorily (via nonferromagnetic tubes to both ears) randomly for each trial.
The stimulus was followed by a visual fixation cross until the end of the trial (2000~ms after the stimulus onset).
MEG signals were measured from twenty healthy volunteers by a 148-channel magnetometer 
(MAGNES 2500 WH, 4D Neuroimaging, San Diego, USA) inside a magnetically shielded room. 
The data were downsampled to 300~Hz, and epoched into trials.
The contaminated trials were rejected by visual inspections, 
and thereafter the blinks, eye movements, and cardiac artifacts were corrected using ICA 
(see \citet{Westner2018} for more details of the preprocessing).
We further band-pass filtered the data between 4~Hz and 125~Hz,
normalized them to have zero-mean and unit variance at the base line period ($-$1,000~ms to 0~ms) for each channel and trial,
and then cropped from $-$300~ms to 2,000~ms after the onset for each trial. 
The dimension of the data was reduced to 30 by PCA.
There were 219.1$\pm$22.4 trials (110.4$\pm$11.5 for auditory and 108.7$\pm$11.9 for visual) for each subject,
and in total, 2,207 auditory and 2,174 visual trials in the whole dataset.

\paragraph{IIA Settings}
We used IIA-TCL for the training, by assuming a third-order NVAR model (NVAR(3)) and the segment-wise-stationarity of the latent innovations.
The trial data were divided into 84 equally sized segments of length of 8 samples (26.7~ms), 
and the segment label was used as the auxiliary variable $\ut$. 
The same segment labels were given across the trials; however,
considering the possible stimulus-specific dynamics of the brain, we assigned different labels for the auditory and visual trials.
In total, there are 168 segments (classes) to be discriminated by MLR.
The network architectures of the MLPs are the same with those in Simulation~2,
except that $\h$ and $\phi$ take $\ve x_{t:t-3}$ and $\ve x_{t-1:t-3}$ as inputs, respectively,
the number of units of each layer was fixed to 30, and that of the last layer (number of components) was 5.
The smaller number of components than the data dimension can be justified by assuming 
the stationarity of the remaining components (the remaining innovations do not depend on $\ve u$; \citet{Hyvarinen2016}).
Considering the fast sampling rate of the data (300~Hz), we fixed the time lag between two consecutive samples to 3 (10~ms). 
The other settings were as in Simulation~2.
The training of a four-layer model by IIA-TCL took about 2~hours 
(Intel Xeon 3.5~GHz 16 core CPUs, 376~GB Memory, NVIDIA Tesla V100 GPU).

\paragraph{Evaluation Methods}
For evaluation, we performed classification of the stimulus modality (auditory or visual) by using the estimated innovations.
The classification was performed using a linear support vector machine (SVM) classifier trained on the stimulation label
and sliding-window-averaged innovations (width=16 and stride=8 samples) obtained for each trial. The performance was evaluated by
the generalizability of a classifier across subjects, i.e., one-subject-out cross-validation (OSO-CV); the feature extractor and the classifier were trained
only from the training subjects, and then applied to the held-out subject. 
The hyperparameters of the SVM were determined by nested OSO-CV without using the test data. 
For comparison, we also applied NICA based on TCL \citep{Hyvarinen2016} and AD-NVAR(3)
(See Supplementary Material~\ref{sec:app_baseline} for the details of the baseline methods,
with changing $\xtm$ to $\ve x_{t-1:t-3}$).
We additionally applied principal component analysis (PCA) to the estimations by AD-NVAR(3) before applying linear ICA
to reduce the dimension to 5 for fair comparisons.
We omitted $L=1$ for IIA-TCL because of the instability of training.

We also visualized the spatial characteristics of each innovation component
by estimating the optimal (maximal and minimal) input $\xt$ while fixing $\ve x_{t-1:t-3}$ to zero.
This method is commonly used in deep learning studies 
to visualize the input specificities of a hidden node of a neural network.
We used $l_2$ regularization on the input to avoid overfitting.

\subsection{Details of the baseline methods}
\label{sec:app_baseline}

\paragraph{NVAR with additive innovation model (AD-NVAR)}
AD-NVAR assumes NVAR with additive innovation model:
\begin{equation}
        \xt = \ve f_\text{ad}(\xtm) + \st,  \label{eq:fad}
\end{equation}
where $\f_\text{ad}:  \R^n \rightarrow \R^n$ is an unknown mixing model, and $\st \in \R^n$ is the latent innovations to be estimated.
For the estimation, we firstly estimate the mixing model from the observable time series,
which can be done practically by training a nonlinear predictor which takes $\xtm$ as an input and then outputs the estimation of $\xt$ so as to minimize the mean squared prediction errors.
The error term was then used as the estimation of the additive innovation $\st$.
Since the obtained components are not guaranteed to be mutually independent,  
we additionally applied linear ICA based on nonstationarity of variance (NSVICA; \citet{Hyvarinen2001})
to the estimated additive innovations for fair comparisons.
For the mixing model $\ve f_\text{ad}$, we used an MLP with the similar architecture as IIA, 
except for the difference of the dimension of the input.

\paragraph{Variational auto encoder (VAE)}
We used VAE \citep{Kingma2014} as a baseline of unsupervised representation learning frameworks.
VAE assumes that the latent variables have spherical Gaussian distribution,
then embed data into the latent space in an unsupervised manner
by training an encoder, which embeds the data into the latent space,
and a decoder, which reconstructs the input from the latent variables, so as to minimize the reconstruction error.
In the simulations, we trained an encoder based on an MLP, which nonlinearly embeds an input $(\xt, \xtm) \in \R^{2n}$ into an $n$-dimensional feature space
representing the estimation of the innovation.
The number of nodes in each layer was designed to linearly decrease from input ($2n$) to the output ($n$).
We additionally applied linear ICA based on nonstationarity of variance (NSVICA; \citet{Hyvarinen2001})
to the estimated innovations for fair comparisons because VAE does not assume the independence on the estimations.

\paragraph{Nonlinear Independent component analysis (NICA)}
NICA assumes instantaneous nonlinear mixture model:
\begin{equation}
	\xt = \f_\text{ICA}(\st),
\end{equation}	
where $\f_\text{ICA}: \R^{n} \rightarrow \R^{n}$ is the instantaneous (nonlinear) mixture function,
and $\st$ is the latent components.
Since general NICA problem has indeterminacy \citep{Hyvarinen1999}, we need some assumptions on the latent components
to guarantee the identifiability, similarly to the IIA frameworks. We here used NICA-GCL \citep{Hyvarinen2019}, NICA-TCL \citep{Hyvarinen2016}, and NICA-HMM \citep{Halva2020}
for comparison, which the basic proofs of IIA are based on.
In the simulations, we estimated the independent components by the similar architecture as IIA (e.g., latent components assumptions, MLPs, and so on),
except that it assumed the instantaneous mixture model for the observation.


\end{document}